\documentclass{article}

\PassOptionsToPackage{numbers}{natbib}

    \usepackage[final]{neurips_2021}

\usepackage[utf8]{inputenc} 
\usepackage[T1]{fontenc}    
\usepackage{hyperref}       
\usepackage{url}            
\usepackage{booktabs}       
\usepackage{amsfonts}       
\usepackage{nicefrac}       
\usepackage{microtype}      
\usepackage{xcolor}         
\usepackage{wrapfig}
\usepackage{caption}
\usepackage[ruled,linesnumbered]{algorithm2e}
\usepackage{setspace}
\usepackage{subcaption}
\usepackage{xr}

\usepackage{amsmath,amsfonts,bm}
\usepackage{multirow}
\usepackage{tabularx}
\usepackage{graphicx}
\usepackage{adjustbox}
\usepackage{amsthm}
\usepackage{xspace}

\renewcommand{\arraystretch}{1.3}
\newcommand{\mset}[1]{\left\{\kern-.5em\left\{ #1 \right\}\kern-.5em\right\}}
\newcommand{\mmset}[1]{\{\kern-.4em\{ #1 \}\kern-.4em\}}

\newcommand{\norm}[1]{\left\Vert#1\right\Vert}

\newcommand{\abs}[1]{\left\vert#1\right\vert}

\newcommand{\set}[1]{\left\{#1\right\}}
\newcommand{\parr}[1]{\left (#1\right )}
\newcommand{\brac}[1]{\left [#1\right ]}

\newcommand{\Real}{\mathbb R}

\newcommand{\eps}{\varepsilon}
\newcommand{\too}{\rightarrow}

\newcommand{\one}{\mathbf{1}}

\newcommand{\eg}{{e.g.}\xspace}
\newcommand{\ie}{{i.e.}\xspace}

 \newtheorem{theorem}{Theorem}
 \newtheorem{lemma}{Lemma}
 \newtheorem{proposition}{Proposition}

\def\eqref#1{equation~\ref{#1}}

\def\1{\bm{1}}

\def\eps{{\epsilon}}

\def\va{{\bm{a}}}
\def\vb{{\bm{b}}}
\def\vc{{\bm{c}}}

\def\vf{{\bm{f}}}

\def\vn{{\bm{n}}}

\def\vv{{\bm{v}}}

\def\vx{{\bm{x}}}

\def\vy{{\bm{y}}}
\def\vz{{\bm{z}}}
\def\vec1{{\bm{1}}}

\DeclareMathAlphabet{\mathsfit}{\encodingdefault}{\sfdefault}{m}{sl}
\SetMathAlphabet{\mathsfit}{bold}{\encodingdefault}{\sfdefault}{bx}{n}

\def\gL{{\mathcal{L}}}
\def\gM{{\mathcal{M}}}

\def\gS{{\mathcal{S}}}
\def\gT{{\mathcal{T}}}

\newcommand{\E}{\mathbb{E}}

\title{Volume Rendering of Neural Implicit Surfaces}

\author{%
  Lior Yariv$^1$ \And 
  Jiatao Gu$^2$ \And
  Yoni Kasten$^1$ \And
  Yaron Lipman$^{1,2}$ \AND
  \vspace{-25pt} \\
  $^1$Weizmann Institute of Science \qquad
  $^2$Facebook AI Research
}

\begin{document}

\maketitle
\begin{abstract}
Neural volume rendering became increasingly popular recently due to its success in synthesizing novel views of a scene from a sparse set of input images. 
So far, the geometry learned by neural volume rendering techniques was modeled using a generic density function. Furthermore, the geometry itself was extracted using an arbitrary level set of the density function leading to a noisy, often low fidelity reconstruction.
The goal of this paper is to improve geometry representation and reconstruction in neural volume rendering. We achieve that by modeling the volume density as a function of the geometry. This is in contrast to previous work modeling the geometry as a function of the volume density. 
In more detail, we define the volume density function as Laplace's cumulative distribution function (CDF) applied to a signed distance function (SDF) representation. This simple density representation has three benefits: (i) it provides a useful inductive bias to the geometry learned in the neural volume rendering process; (ii) it facilitates a bound on the opacity approximation error, leading to an accurate sampling of the viewing ray. Accurate sampling is important to provide a precise coupling of geometry and radiance; and (iii) it allows efficient unsupervised disentanglement of shape and appearance in volume rendering. 
%
%
Applying this new density representation to challenging scene multiview datasets produced high quality geometry reconstructions, outperforming relevant baselines. Furthermore, switching shape and appearance between scenes is possible due to the disentanglement of the two. 
\end{abstract}

\section{Introduction}

Volume rendering~\citep{max1995optical} is a set of techniques that renders volume \emph{density} in \emph{radiance fields} by the so called volume rendering integral. 
It has recently been shown that representing both the density and radiance fields as neural networks can lead to excellent prediction of novel views by learning only from a sparse set of input images. This neural volume rendering approach, presented in \cite{mildenhall2020nerf} and developed by its follow-ups~\citep{nerv2021,boss2020nerd} approximates the integral as alpha-composition in a differentiable way, allowing to learn simultaneously both from input images. Although this coupling indeed leads to good generalization of novel viewing directions, the density part is not as successful in faithfully predicting the scene's actual geometry, often producing noisy, low fidelity geometry approximation. 

We propose VolSDF to devise a different model for the density in neural volume rendering, leading to better approximation of the scene's geometry while maintaining the quality of view synthesis.
The key idea is to represent the density as a function of the signed distance to the scene's surface, see Figure \ref{fig:teas}. 
Such density function enjoys several benefits. 
First, it guarantees the existence of a well-defined surface that generates the density. This provides a useful inductive bias for disentangling density and radiance fields, which in turn provides a more accurate geometry approximation. 
Second, we show this density formulation allows bounding the approximation error of the opacity along rays. This bound is used to sample the viewing ray so to provide a faithful coupling of density and radiance field in the volume rendering integral. 
E.g., without such a bound the computed radiance along a ray (pixel color) can potentially miss or extend surface parts leading to incorrect radiance approximation. 

A closely related line of research, often referred to as neural implicit surfaces~\citep{niemeyer2019differentiable,yariv2020multiview,Kellnhofer:2021:nlr}, have been focusing on representing the scene's geometry implicitly using a neural network, making the surface rendering process differentiable. 
The main drawback of these methods is their requirement of masks that separate objects from the background. Also, learning to render surfaces directly tends to grow extraneous parts due to optimization problems, which are avoided by volume rendering. In a sense, our work combines the best of both worlds: \textit{volume rendering with neural implicit  surfaces}.


We demonstrate the efficacy of VolSDF by reconstructing surfaces from the DTU~\citep{jensen2014large} and Blended-MVS~\citep{yao2020blendedmvs} datasets. VolSDF produces more accurate surface reconstructions compared to NeRF~\citep{mildenhall2020nerf} and NeRF++~\citep{kaizhang2020}, and comparable reconstruction compared to IDR~\citep{yariv2020multiview}, while avoiding the use of object masks. Furthermore, we show disentanglement results with our method, \ie, switching the density and radiance fields of different scenes, which is shown to fail in NeRF-based models.\vspace{-3pt}

\begin{figure}\hspace{-12pt}
\centering
    \includegraphics[width=\textwidth]{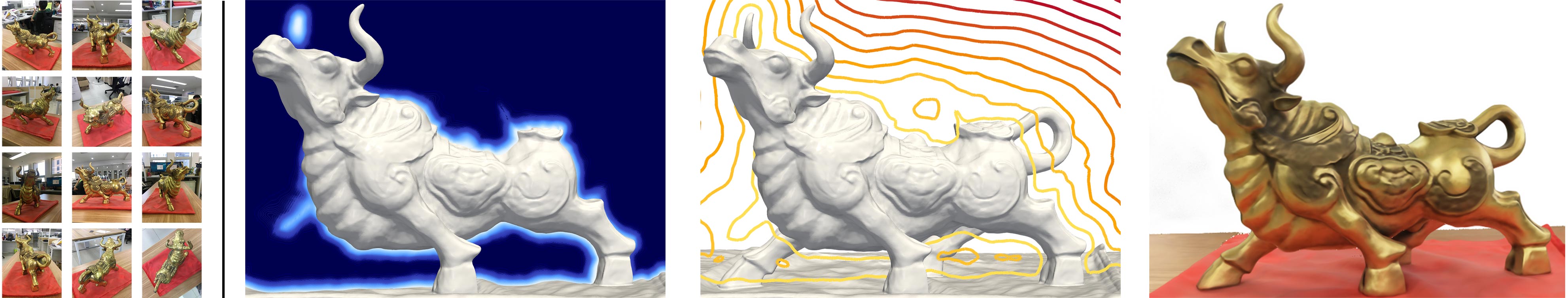}
    \caption{VolSDF: given a set of input images (left) we learn a volumetric density (center-left, sliced) defined by a signed distance function (center-right, sliced) to produce a neural rendering (right). This definition of density facilitates high quality geometry reconstruction (gray surfaces, middle).} \label{fig:teas}
    \vspace{-15pt}
\end{figure}

\section{Related work}\vspace{-2pt}
\textbf{Neural Scene Representation \& Rendering}
Implicit functions are traditionally adopted in modeling 3D scenes~\citep{10.1145/2508363.2508374,izadi2011kinectfusion,dai2017bundlefusion}. 
Recent studies have been focusing on model implicit functions with multi-layer perceptron (MLP) due to its expressive representation power and low memory foot-print, including scene (geometry \& appearance) representation~\citep{genova2019learning,michalkiewicz2019implicit,mescheder2019occupancy,niemeyer2019occupancy,oechsle2019texture,peng2020convolutional,xu2019disn,park2019deepsdf,takikawa2021neural} and free-view rendering~\citep{sitzmann2019scene,liu2019dist,saito2019pifu,oechsle2020learning,lombardi2019neural,mildenhall2020nerf,liu2020neural,kaizhang2020,nerv2021,boss2020nerd}. 
In particular, NeRF~\citep{mildenhall2020nerf} has opened up a line of research (see \cite{dellaert2020neural} for an overview) combining neural implicit functions together with volume rendering to achieve photo-realistic rendering results. However, it is non-trivial to find a proper threshold to extract surfaces from the predicted density, and the recovered geometry is far from satisfactory. Furthermore, sampling of points along a ray for rendering a pixel is done using an opacity function that is approximated from another network without any guarantee for correct approximation.

\textbf{Multi-view 3D Reconstruction}
Image-based 3D surface reconstruction (multi-view stereo) has been a longstanding problem in the past decades. Classical multi-view stereo approaches are generally either depth-based~\citep{Barnes:2009:PAR,schoenberger2016mvs,galliani2015massively,5226635} or voxel-based~\citep{de1999poxels,937544,seitz1999photorealistic}. For instance, in COLMAP~\citep{schoenberger2016mvs} (a typical depth-based method) image features are extracted and matched across different views to estimate depth. Then the predicted depth maps are fused to obtain dense point clouds. To obtain the surface, an additional meshing step e.g. Poisson surface reconstruction~\citep{SGP:SGP06:061-070} is applied.   
However, these methods with complex pipelines may accumulate errors at each stage and usually result in incomplete 3D models, especially for non-Lambertian surfaces as they can not handle view dependent colors. On the contrary, although it produces complete models by directly modeling objects in a volume, voxel-based approaches are limited to low resolution due to high memory consumption.  
Recently, neural-based approaches such as DVR~\citep{niemeyer2019differentiable}, IDR~\citep{yariv2020multiview}, NLR~\citep{Kellnhofer:2021:nlr} have also been proposed to reconstruct scene geometry from multi-view images. However, these methods require accurate object masks and appropriate weight initialization due to the difficulty of propagating gradients. 

Independently from and concurrently with our
work here, \cite{oechsle2021unisurf} also use implicit surface representation incorporated into volume rendering. In particular, they replace the local transparency function with an occupancy network~\citep{mescheder2019occupancy}. This allows adding surface smoothing term to the loss, improving the quality of the resulting surfaces. Differently from their approach, we use signed distance representation, regularized with an Eikonal loss~\citep{yariv2020multiview,gropp2020implicit} without any explicit smoothing term. Furthermore, we show that the choice of using signed distance allows bounding the opacity approximation error, facilitating the approximation of the volume rendering integral for the suggested family of densities.

\section{Method}\vspace{-5pt}
\label{s:method}

In this section we introduce a novel parameterization for volume density, defined as transformed signed distance function. Then we show how this definition facilitates the volume rendering process. In particular, we derive a bound of the error in the opacity approximation and consequently devise a sampling procedure for approximating the volume rendering integral.

\subsection{Density as transformed SDF} 
Let the set $\Omega\subset\Real^3$ represent the space occupied by some object in $\Real^3$, and $\gM=\partial \Omega$ its boundary surface. We denote by $\one_\Omega$ the $\Omega$ indicator function, and by $d_\Omega$ the Signed Distance Function (SDF) to its boundary $\gM$,
\begin{equation}
\one_\Omega(\vx) = \begin{cases} 1 & \text{if } \vx\in\Omega \\ 0 & \text{if }\vx\notin \Omega \end{cases}, \quad \text{and }\  d_\Omega(\vx) = (-1)^{\one_\Omega(\vx)}\min_{\vy\in\gM}\norm{\vx-\vy},
\end{equation}
where $\norm{\cdot}$ is the standard Euclidean 2-norm.
In neural volume rendering the volume \emph{density} $\sigma:\Real^3\too\Real_+$ is a scalar volumetric function, where $\sigma(\vx)$ is the rate that light is occluded at point $\vx$; $\sigma$ is called density since it is proportional to the particle count per unit volume at $\vx$~\citep{max1995optical}. In previous neural volumetric rendering approaches~\citep{mildenhall2020nerf,liu2020neural,kaizhang2020}, the density function, $\sigma$, was modeled with a general-purpose Multi-Layer Perceptron (MLP). In this work we suggest to model the density using a certain transformation of a learnable Signed Distance Function (SDF) $d_\Omega$, 
namely
\begin{equation}\label{e:density}
\sigma(\vx) = \alpha \Psi_\beta \parr{-d_\Omega(\vx)},
\end{equation}
where $\alpha, \beta>0$ are learnable parameters, and $\Psi_\beta$ is the Cumulative Distribution Function (CDF) of the Laplace distribution with zero mean and $\beta$ scale (\ie, mean absolute deviation, which is intuitively  the $L_1$ version of the standard deviation),
\begin{equation}\label{e:laplace}
\Psi_\beta(s) = \begin{cases} \frac{1}{2} \exp\parr{\frac{s}{\beta}} & \text{if } s\leq 0 \\
1-\frac{1}{2}\exp\parr{-\frac{s}{\beta}} & \text{if } s>0
\end{cases}
\end{equation}
Figure \ref{fig:teas} (center left and right) depicts an example of such a density and SDF. 
As can be readily checked from this definition, as $\beta$ approach zero, the density $\sigma$ converges to a scaled indicator function of $\Omega$, that is $\sigma \too \alpha\one_\Omega$ for all points $\vx\in\Omega\setminus\gM$.

Intuitively, the density $\sigma$ models a homogeneous object with a constant density $\alpha$ that smoothly decreases near the object's boundary, where the smoothing amount is controlled by $\beta$.
The benefit in defining the density as in \eqref{e:density} is two-fold: First, it provides a useful inductive bias for the surface geometry $\gM$, and provides a principled way to reconstruct the surface, \ie, as the zero level-set of $d_\Omega$. This is in contrast to previous work where the reconstruction was chosen as an arbitrary level set of the learned density. Second, the particular form of the density as defined in \eqref{e:density} facilitates a bound on the error of the \emph{opacity} (or, equivalently the \emph{transparency}) of the rendered volume, a crucial component in the volumetric rendering pipeline. In contrast, such a bound will be hard to devise for a generic MLP densities.

\subsection{Volume rendering of $\sigma$} \label{ss:volume_rendering_of_sigma}
In this section we review the volume rendering integral and the numerical integration commonly used to approximate it, requiring a set $\gS$ of sample points per ray. In the following section (Section \ref{ss:bound}), we explore the properties of the density $\sigma$ and derive a bound on the opacity approximation error along viewing rays. Finally, in Section \ref{s:sampling_algorith} we derive an algorithm for producing a sample $\gS$ to be used in the volume rendering numerical integration. 

In volume rendering we consider a ray $\vx$ emanating from a camera position $\vc\in\Real^3$ in direction $\vv\in\Real^3$, $\norm{\vv}=1$, defined by $\vx(t)=\vc+t\vv$, $t\geq 0$.
In essence, volume rendering is all about approximating the integrated (\ie, summed) light radiance along this ray reaching the camera. There are two important quantities that participate in this computation: the volume's \emph{opacity} $O$, or equivalently, its \emph{transperancy} $T$, and the \emph{radiance field} $L$. 

The \emph{transparency} function of the volume along a ray $\vx$, denoted $T$, indicates, for each $t\geq 0$, the probability a light particle succeeds traversing the segment $[\vc,\vx(t)]$ without bouncing off,  
\begin{equation}\label{e:T}
T(t) = \exp\parr{ - \int_0^t \sigma(\vx(s)) ds },
\end{equation}

\begin{wrapfigure}[23]{r}{0.4\textwidth}
   \centering
   \vspace{-0pt}
    \begin{tabular}{@{\hskip0.0pt}c@{\hskip0.0pt}}
    \includegraphics[width=0.4\textwidth]{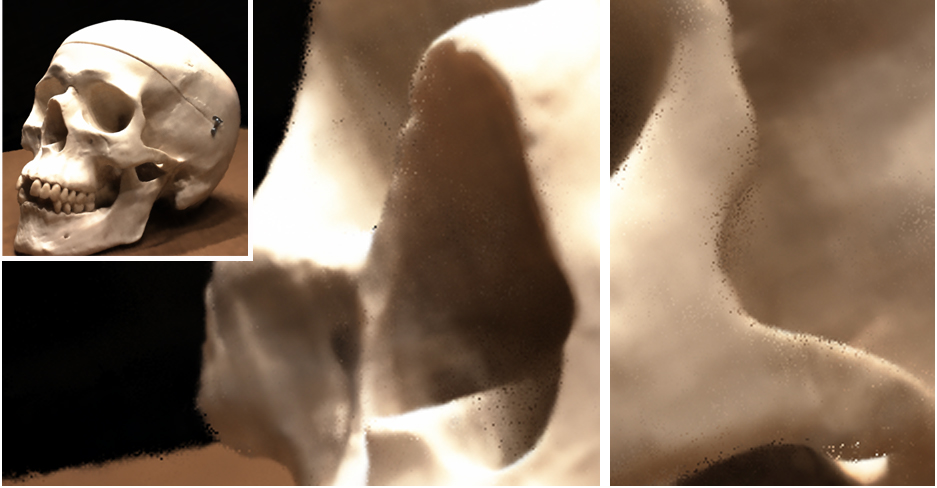} \\ 
    NeRF \\
     \includegraphics[width=0.4\textwidth]{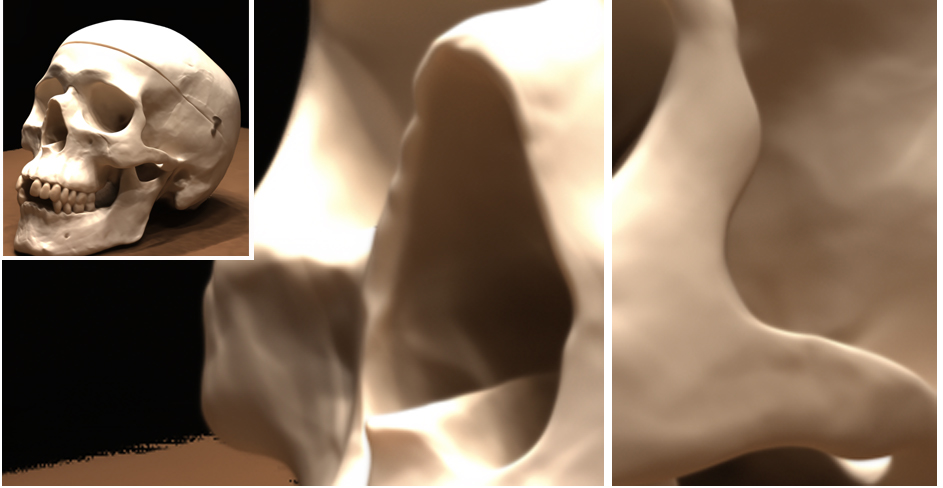} \\
    \textbf{VolSDF}
    \end{tabular}
    \vspace{-0pt}
    \caption{Qualitative comparison to NeRF. VolSDF shows less artifacts.\vspace{-0pt}} \label{fig:rendering_compare}
   \end{wrapfigure}
and the \emph{opacity} $O$ is the complement probability, 
\begin{equation}\label{e:O}
O(t)=1-T(t).
\end{equation}
Note that $O$ is a monotonic increasing function where $O(0)=0$, and assuming that every ray is eventually occluded $O(\infty)=1$. In that sense we can think of $O$ as a CDF, and 
\begin{equation}\label{e:tau}
\tau(t)=\frac{dO}{dt}(t) = \sigma(\vx(t))  T(t)
\end{equation}
is its Probability Density Function (PDF). The volume rendering equation is the expected light along the ray, 
\begin{equation}\label{e:vol}
I(\vc,\vv) = \int_0^\infty  L(\vx(t),\vn(t),\vv) \tau(t) dt ,
\end{equation}
where $L(\vx,\vn,\vv)$ is the radiance field, namely the amount of light emanating from point $\vx$ in direction $\vv$; in our formulation 
we also allow $L$ to depend on the level-set's normal, \ie, $\vn(t)=\nabla_\vx d_\Omega(\vx(t))$. Adding this dependency is motivated by the fact that BRDFs of common materials are often encoded with respect to the surface normal, facilitating disentanglement as done in surface rendering \citep{yariv2020multiview}. We will get back to disentanglement in the experiments section. 
The integral in \eqref{e:vol} is approximated using a numerical quadrature, namely the rectangle rule, at some discrete samples $\gS=\set{s_i}_{i=1}^m$, $0=s_1<s_2<\ldots<s_m=M$, where $M$ is some large constant: 
\begin{equation}\label{e:numerical_int}
  I(\vc,\vv) \approx \hat{I}_\gS(\vc,\vv) = \sum_{i=1}^{m-1} \hat{\tau}_i L_i,  
\end{equation}
where we use the subscript $\gS$ in $\hat{I}_\gS$ to highlight the dependence of the approximation on the sample set $\gS$, $\hat{\tau}_i\approx \tau(s_i)\Delta s$ is the approximated PDF multiplied by the interval length, and $L_i=L(\vx(s_i),\vn(s_i),\vv)$ is the sampled radiance field. We provide full derivation and detail of $\hat{\tau}_i$ in the supplementary.

\textbf{Sampling.} Since the PDF $\tau$ is typically extremely concentrated near the object's boundary (see \eg, Figure \ref{fig:algo_conv}, right) the choice of the sample points $\gS$ has a crucial effect on the approximation quality of \eqref{e:numerical_int}. 
One solution is to use an adaptive sample, \eg, $\gS$ computed with the inverse CDF, \ie, $O^{-1}$. However, $O$ depends on the density model $\sigma$ and is not given explicitly. In \cite{mildenhall2020nerf} a second, coarse network was trained specifically for the approximation of the opacity $O$, and was used for inverse sampling. However, the second network's density does not necessarily faithfully represents the first network's density, for which we wish to compute the volume integral. Furthermore, as we show later, one level of sampling could be insufficient to produce an accurate sample $\gS$. 
Using a naive or crude approximation of $O$ would lead to a sub-optimal sample set $\gS$ that misses, or over extends non-negligible $\tau$ values. Consequently, incorrect radiance approximations can occur (\ie, pixel color), potentially harming the learned density-radiance field decomposition. Our solution works with a single density $\sigma$, and the sampling $\gS$ is computed by a sampling algorithm based on an error bound for the opacity approximation. Figure \ref{fig:rendering_compare} compares the NeRF and VolSDF renderings for the same scene. Note the salt and pepper artifacts in the NeRF rendering caused by the random samples; using fixed (uniformly spaced) sampling in NeRF leads to a different type of artifacts shown in the supplementary.

\subsection{Bound on the opacity approximation error}\label{ss:bound}

In this section we develop a bound on the opacity approximation error using the rectangle rule. 
For a set of samples $\gT=\set{t_i}_{i=1}^n$, $0=t_1<t_2<\cdots<t_n=M$, we let $\delta_i=t_{i+1}-t_{i}$, and $\sigma_i=\sigma(\vx(t_i))$. Given some $t\in (0,M]$, assume $t\in [t_k,t_{k+1}]$, and apply the rectangle rule (\ie, left Riemann sum) to get the approximation:
\begin{equation}\label{e:rectangle_rule}
\int_0^t \sigma(\vx(s))ds = \widehat{R}(t) +E(t), \quad \text{where }\ \widehat{R}(t) = \sum_{i=1}^{k-1}\delta_i \sigma_i + (t-t_{k})\sigma_k
\end{equation}
is the rectangle rule approximation, and $E(t)$ denotes the error in this approximation. The corresponding approximation of the opacity function (\eqref{e:O}) is
\begin{equation}\label{e:O_rect}
\widehat{O}(t) = 1- \exp\parr{-\widehat{R}(t)}.
\end{equation}
Our goal in this section is to derive a uniform bound over $[0,M]$ to the approximation $\widehat{O}\approx O$. The key is the following bound on the derivative\footnote{As $d_\Omega$ is not differentiable everywhere the bound is on the \emph{Lipschitz constant} of $\sigma$, see supplementary.} of the density $\sigma$ inside an interval along the ray $\vx(t)$:
%
%
\begin{theorem}\label{thm:density_der}
The derivative of the density $\sigma$ within a segment $[t_i,t_{i+1}]$ satisfies
\begin{equation}\label{e:dsigma_bound}
   \abs{\frac{d}{ds}\sigma(\vx(s))} \leq \frac{\alpha}{2\beta}\exp\parr{-\frac{d^\star_i}{\beta}},\  \text{where } d_i^\star = \min_{\substack{s\in [t_i,t_{i+1}]\\ \vy \notin B_i\cup B_{i+1}}} \norm{\vx(s)-\vy}
    ,\end{equation} 
    and $B_i=\set{\vx \ \vert \ \norm{\vx-\vx(t_i)}<|d_i|}$, $d_i=d_\Omega(\vx(t_i))$. 
\end{theorem}


\begin{wrapfigure}[7]{r}{0.3\textwidth}
\vspace{-33pt}
     \includegraphics[width=0.3\textwidth]{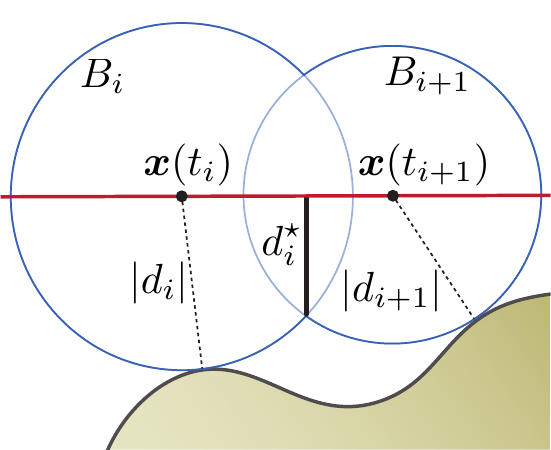}
 \label{fig:d_star}
   \end{wrapfigure}
The proof of this theorem, which is provided in the supplementary, makes a principled use of the signed distance function's unique properties; the explicit formula for $d^*_i$ is a bit cumbersome and therefore is deferred to the supplementary as-well. The inset depicts the boundary of the open balls union $B_i \cup B_{i+1}$, the interval $[\vx(t_i),\vx(t_{i+1})]$ and the bound is defined in terms of the minimal distance between these two sets, \ie, $d^*_i$.

The benefit in Theorem \ref{thm:density_der} is that it allows to bound the density's derivative in each interval $[t_i,t_{i-1}]$ based only on the unsigned distance at the interval's end points, $|d_i|, |d_{i+1}|$, and the density parameters $\alpha,\beta$. This bound can be used to derive an error bound for the rectangle rule's approximation of the opacity,
\begin{equation}\label{e:E_bound}
\abs{E(t)}\leq \widehat{E}(t) = \frac{\alpha}{4\beta}\parr{\sum_{i=1}^{k-1} \delta_i^2 e^{-\frac{d^\star_i}{\beta}} + (t-t_k)^2 e^{-\frac{d^\star_k}{\beta}}}.
\end{equation}
Details are in the supplementary. Equation \ref{e:E_bound} leads to the following opacity error bound, also proved in the supplementary:
\begin{theorem}\label{thm:bound}
For $t\in[0,M]$, the error of the approximated opacity $\hat{O}$ can be bounded as follows:
\begin{equation}\label{e:bound_O}
\abs{O(t)-\widehat{O}(t)} \leq \exp\parr{-\widehat{R}(t)}\parr{\exp\parr{\widehat{E}(t)}-1} 
\end{equation}
\end{theorem}
Finally, we can bound the opacity error for $t\in[t_k,t_{k+1}]$ by noting that $\widehat{E}(t)$, and consequently also $\exp(\widehat{E}(t))$ are monotonically increasing in $t$, while $\exp(-\widehat{R}(t))$ is monotonically decreasing in $t$, and therefore
\begin{equation}\label{e:bound_per_interval}
\max_{t\in[t_k,t_{k+1}]}\abs{O(t)-\widehat{O}(t)} \leq \exp\parr{-\widehat{R}(t_k)} \parr{\exp(\widehat{E}(t_{k+1}))-1}.
\end{equation}
Taking the maximum over all intervals furnishes a bound $B_{\gT,\beta}$ as a function of $\gT$ and $\beta$,
\begin{equation}\label{e:bound}
\max_{t\in[0,M]}\abs{O(t)-\widehat{O}(t)} \leq B_{\gT,\beta} = \max_{k\in [n-1]} \set{\exp\parr{-\widehat{R}(t_k)} \parr{\exp(\widehat{E}(t_{k+1}))-1}},
\end{equation}
where by convention $\widehat{R}(t_0)=0$, and $[\ell]=\set{1,2,\ldots,\ell}$. See Figure \ref{fig:algo_conv}, where this bound is visualized in faint-red. 

To conclude this section we derive two useful properties, proved in the supplementary. The first, is that sufficiently dense sampling is guaranteed to reduce the error bound $B_{\gT,\eps}$:
\begin{lemma}\label{lem:dense}
Fix $\beta>0$. For any $\eps>0$ a sufficient dense sampling $\gT$ will provide $B_{\gT,\beta}<\eps$.
\end{lemma}
Second, with a fixed number of samples we can set $\beta$ such that the error bound is below $\eps$: 
\begin{lemma}\label{lem:beta_plus}
Fix $n>0$. For any $\eps>0$ a sufficiently large $\beta$ that satisfies
\begin{equation}
        \label{e:beta_0}
  \beta \geq  \frac{\alpha M^2}{4(n-1)\log(1+\eps)} 
    \end{equation}   
    will provide $B_{\gT,\beta}\leq \eps$.
\end{lemma}

\begin{figure}[t]
\centering

\begin{tabular}{@{\hskip0.0pt}l@{\hskip10.0pt}c@{\hskip5.0pt}c@{\hskip5.0pt}c@{\hskip0.0pt}}
\includegraphics[width=0.4\textwidth]{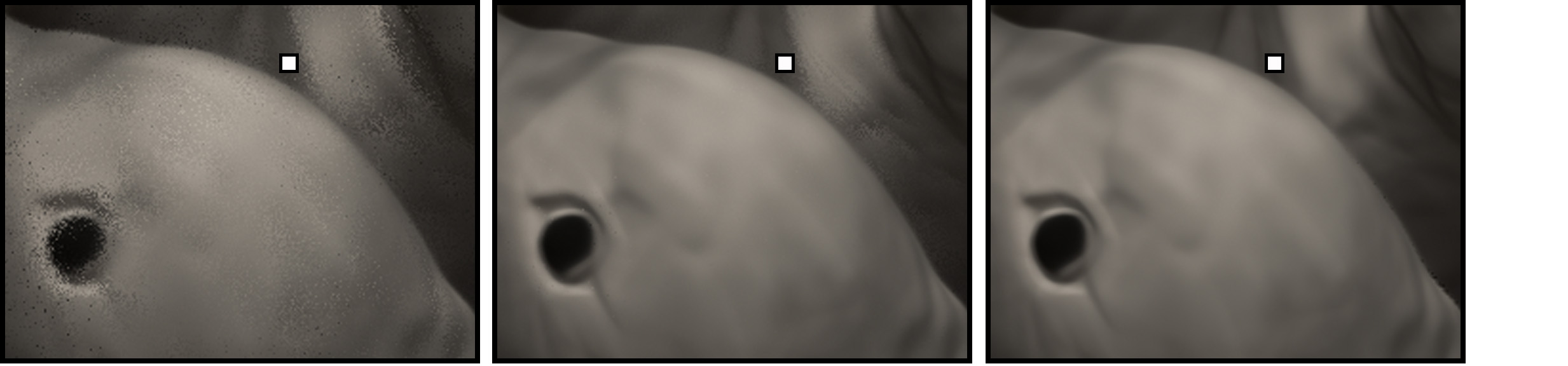}
& {\rotatebox[origin=l]{90}{\scriptsize{Opacity}}}
\includegraphics[width=0.17\textwidth]{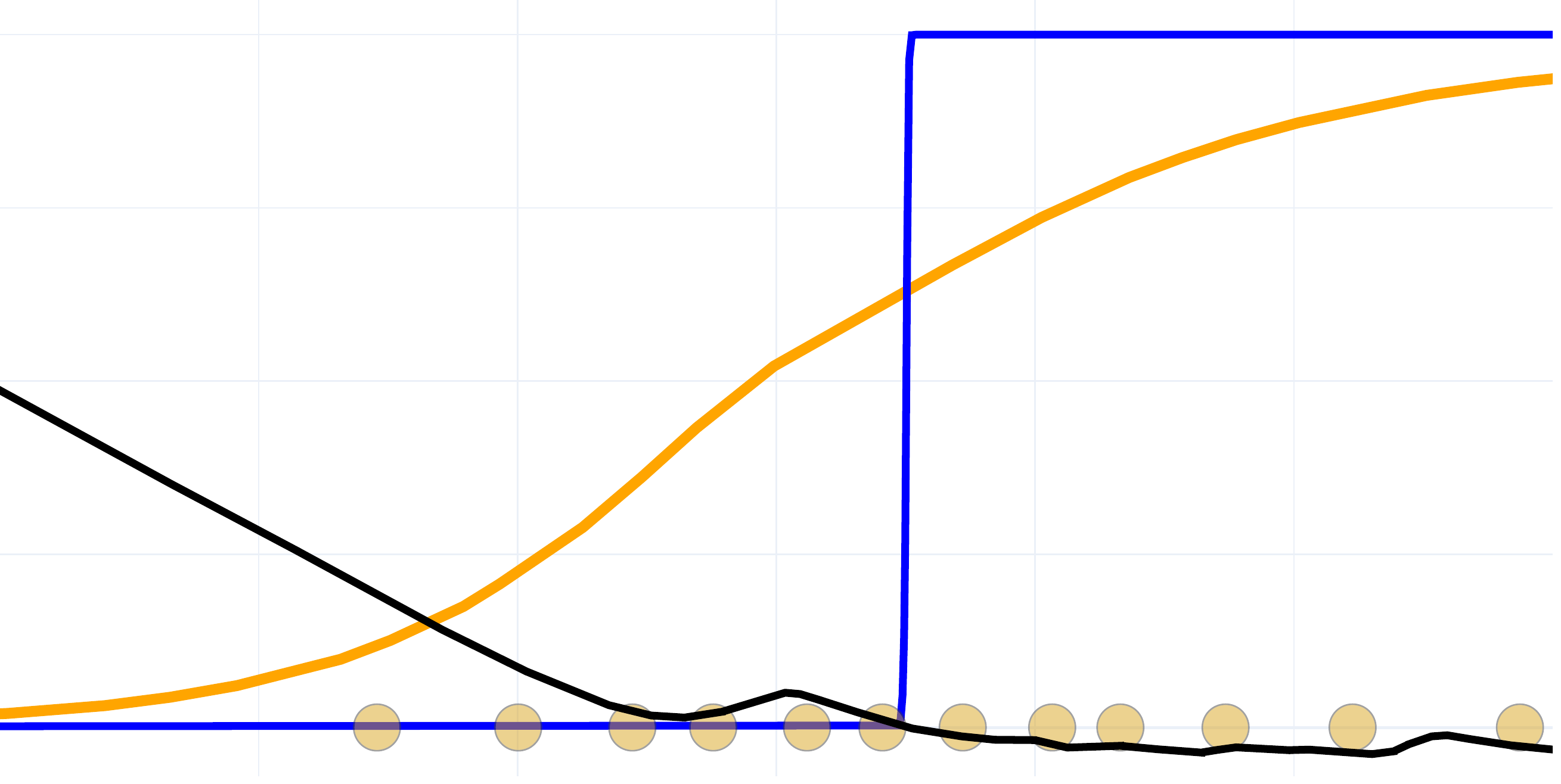} 
&
\includegraphics[width=0.17\textwidth]{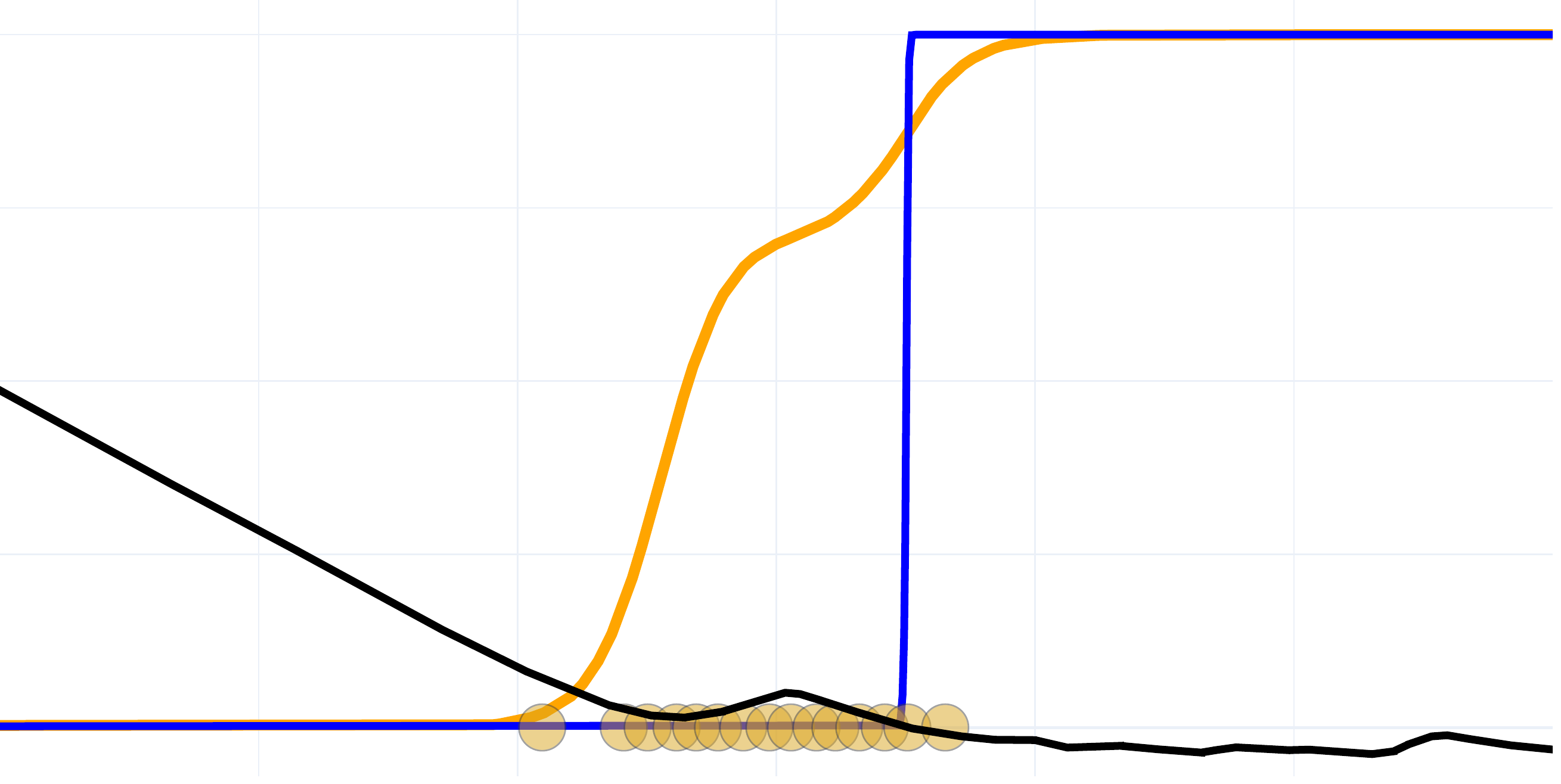} 
&
 \includegraphics[width=0.17\textwidth]{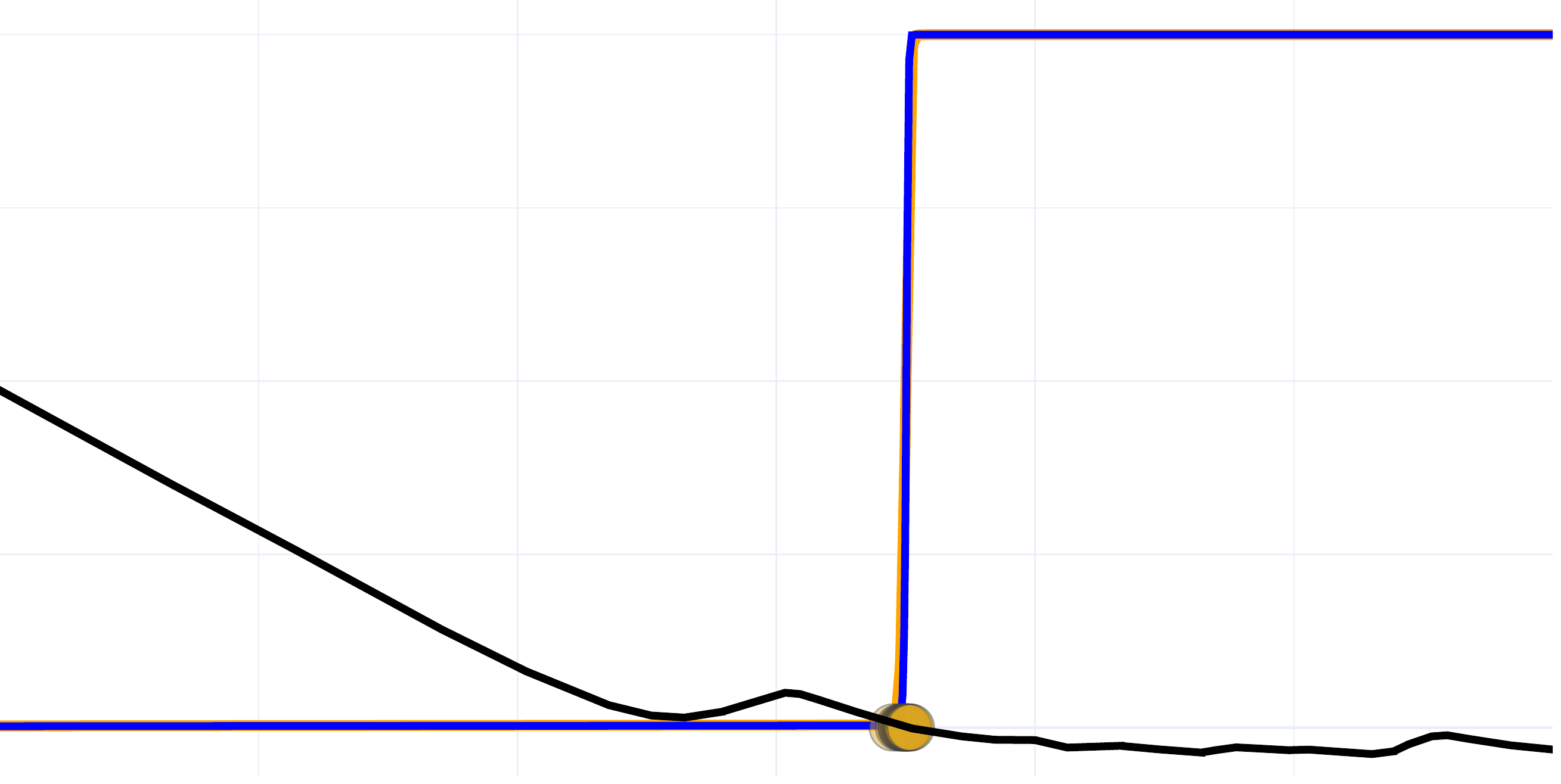} 
\\ 
\includegraphics[width=0.4\textwidth]{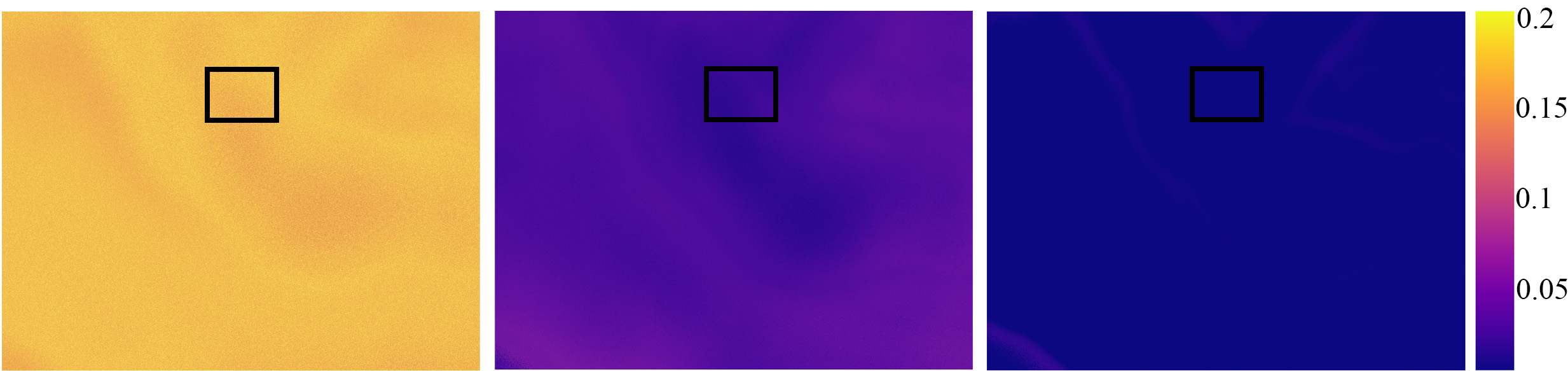}
& {\rotatebox[origin=l]{90}{\scriptsize{Error}}}
\includegraphics[width=0.17\textwidth]{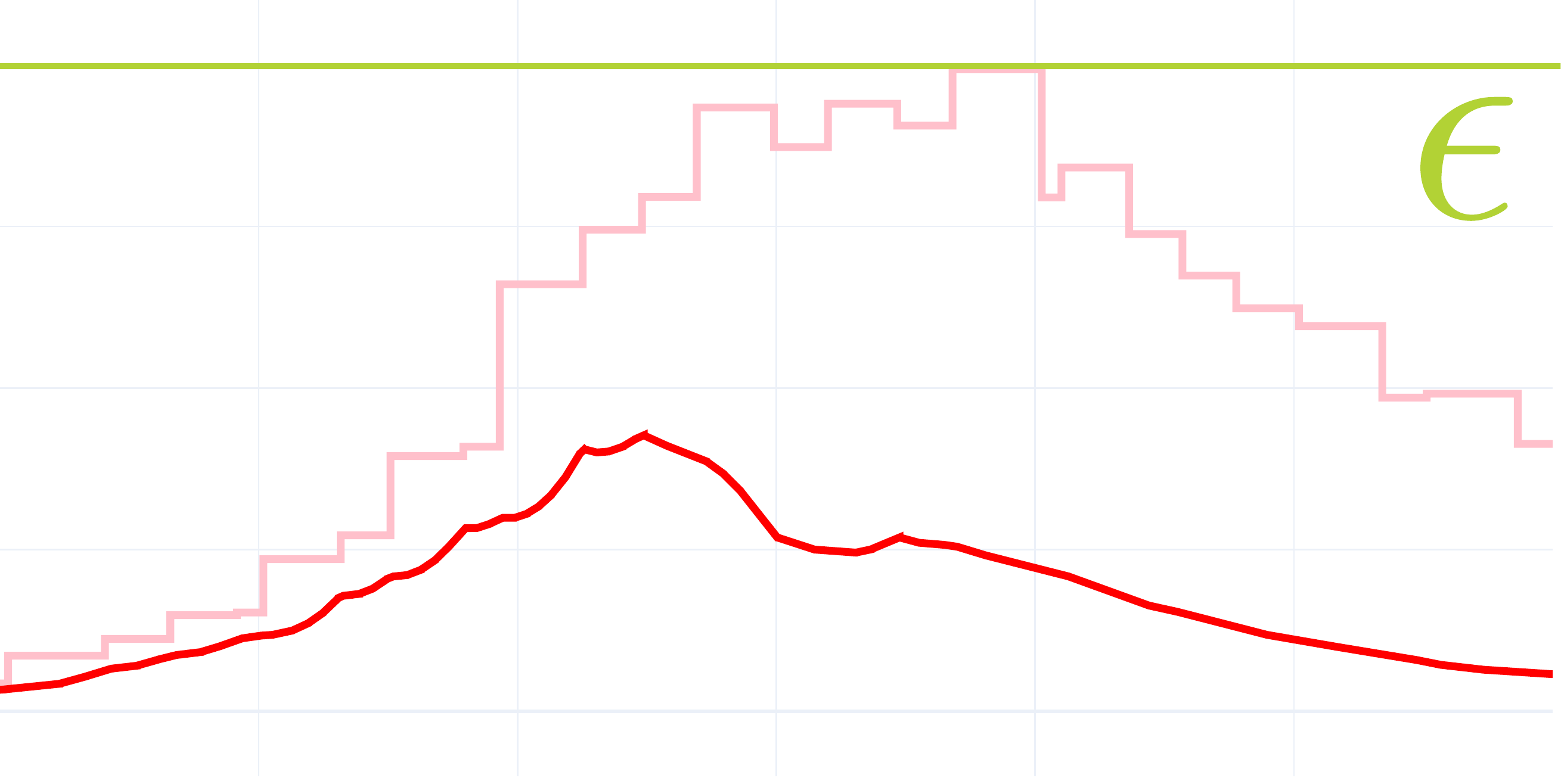}   
&      
\includegraphics[width=0.17\textwidth]{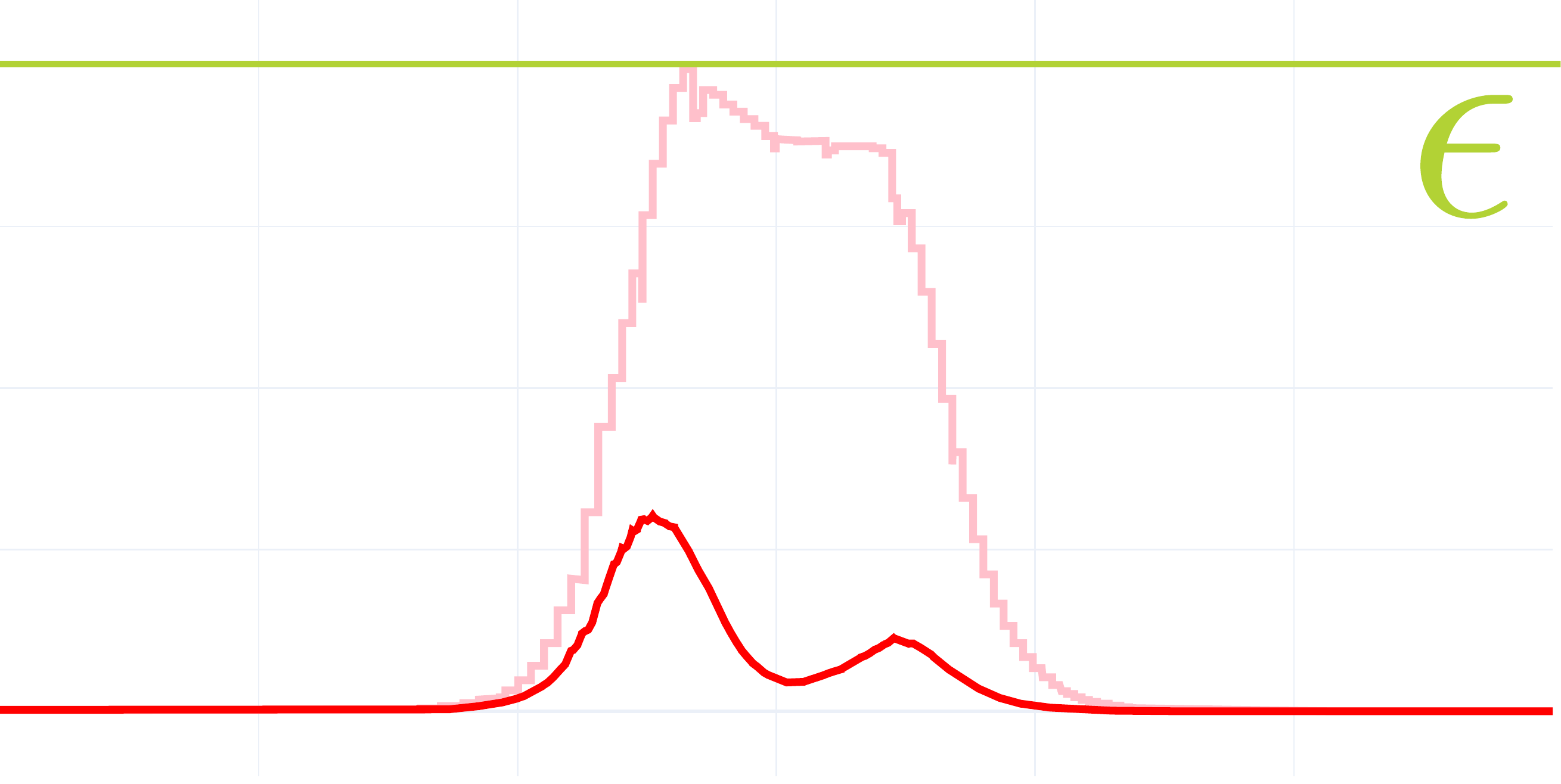} 
&
\includegraphics[width=0.17\textwidth]{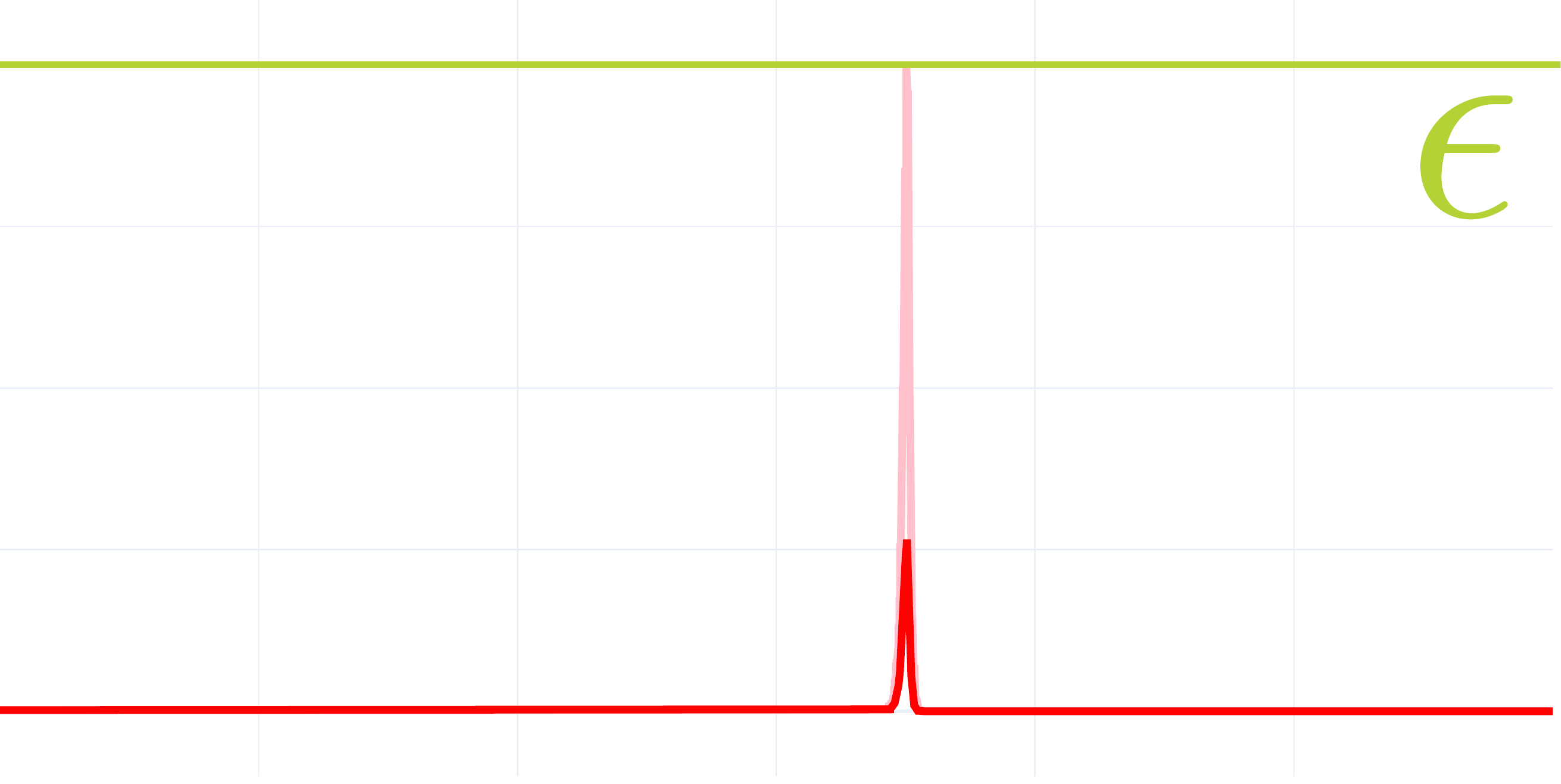} 
\\ 
   \ \  iteration 1 \ \ \   iteration 2 \ \ \  iteration 5 &       iteration 1  & iteration 2 & iteration 5
    \end{tabular}

    \caption{Qualitative evaluation of Algorithm~\ref{alg:meta} after 1, 2 and 5 iterations. Left-bottom: per-pixel $\beta_+$ heatmap; Left-top: rendering of areas marked with black squares. Right-top: for a single ray indicated by white pixel we show the approximated (orange), true opacity (blue), the SDF (black), and $\widehat{O}^{-1}$ sample example (yellow dots). Right-bottom: for the same ray we now show the true opacity error (red), and error bound (faint red). After 5 iterations most of the rays converged, as can be inspected by the blue colors in the heatmap, providing a guaranteed $\eps$ approximation to the opacity, resulting in a crisp and more accurate rendering (center-left, top). \vspace{-10pt} } \label{fig:algo_conv}
\end{figure}
\subsection{Sampling algorithm}\label{s:sampling_algorith}
In this section we develop an algorithm for computing the sampling $\gS$ to be used in \eqref{e:numerical_int}. This is done by first utilizing the bound in \eqref{e:bound} to find samples $\gT$ so that $\widehat{O}$ (via \eqref{e:O_rect}) provides an $\epsilon$ approximation to the true opacity $O$, where $\epsilon$ is a hyper-parameter, that is $B_{\gT,\beta}<\eps$. Second, we perform inverse CDF sampling with $\hat{O}$, as described in Section \ref{ss:volume_rendering_of_sigma}.

Note that from Lemma \ref{lem:dense} it follows that we can simply choose large enough $n$ to ensure $B_{\gT,\beta}<\eps$. However, this would lead to prohibitively large number of samples. Instead, we suggest a simple algorithm to reduce the number of required samples in practice and allows working with a limited budget of sample points. 
In a nutshell, we start with a uniform sampling $\gT=\gT_0$, and use Lemma \ref{lem:beta_plus} to initially set a  $\beta_+>\beta$ that satisfies $B_{\gT,\beta_+}\leq\eps$. Then, we repeatedly upsample $\gT$ to reduce $\beta_+$ while maintaining $B_{\gT,\beta_+}\leq\eps$. Even though this simple strategy is not guaranteed to converge, we find that $\beta_+$ usually converges to $\beta$ (typically $85\%$, see also Figure \ref{fig:algo_conv}), and even in cases it does not, the algorithm provides $\beta_+$ for which the opacity approximation still maintains an $\eps$ error. The algorithm is presented below (Algorithm \ref{alg:meta}).
%
%

We initialize $\gT$ (Line 1 in Algorithm \ref{alg:meta}) with uniform sampling $\gT_0=\set{t_i}_{i=1}^n$, where $t_k = (k-1)\frac{M}{n-1}$, $k\in [n]$ (we use $n=128$ in our implementation). 
Given this sampling we next pick $\beta_+>\beta$ according to Lemma \ref{lem:beta_plus} so that the error bound satisfies the required $\eps$ bound (Line 2 in Algorithm \ref{alg:meta}). 

\begin{wrapfigure}[20]{r}{0.44\textwidth}
\vspace{-15pt}
    \begin{minipage}{0.41\textwidth}
\hspace{10pt}
      \begin{algorithm}[H]\label{alg:meta}
 \caption{Sampling algorithm.}
\setstretch{1.3}
\SetAlgoLined
\KwIn{error threshold $\eps>0$; $\beta$}
Initialize $\gT=\gT_0$\\
Initialize $\beta_+$ such that $B_{\gT,\beta_+}\leq\eps$\\
\While{$B_{\gT,\beta} > \eps$ and not \text{max\_iter} }{
upsample $\gT$ \\
\If{$B_{\gT,\beta_+}<\eps$ }
{
Find $\beta_\star\in (\beta,\beta_+)$ so that $B_{\gT,\beta_\star}=\eps$\\
Update $\beta_+\leftarrow \beta_\star$ \\}
}
Estimate $\widehat{O}$ using $\gT$ and $\beta_+$\\
$\gS \leftarrow $ get fresh $m$ samples using $\hat{O}^{-1}$ \\ 
\Return $\gS$
\end{algorithm}
    \end{minipage}
  \end{wrapfigure}
  
In order to reduce $\beta_+$ while keep $B_{\gT,\beta_+}\leq\eps$, $n$ samples are added to $\gT$ (Line 4 in Algorithm \ref{alg:meta}), where the number of points sampled from each interval is proportional to its current error bound, \eqref{e:bound_per_interval}. 

Assuming $\gT$ was sufficiently upsampled and satisfy $B_{\gT,\beta_+}< \eps$, we decrease $\beta_+$ towards $\beta$.
Since the algorithm did not stop we have that $B_{\gT,\beta}> \eps$. Therefore the Mean Value Theorem implies the existence of $\beta_\star\in (\beta,\beta_+)$ such that $B_{\gT,\beta_\star}= \eps$. We use the bisection method (with maximum of 10 iterations) to efficiently search for $\beta_\star$ and update $\beta_+$ accordingly  (Lines 6 and 7 in Algorithm \ref{alg:meta}).
The algorithm runs iteratively until $B_{\gT,\beta} \leq \eps$ or a maximal number of $5$ iterations is reached. Either way, we use the final $\gT$ and $\beta_+$ (guaranteed to provide $B_{\gT,\beta_+}\leq \eps$) to estimate the current opacity $\widehat{O}$, Line 10 in Algorithm \ref{alg:meta}). Finally we return a fresh set of $m=64$ samples $\hat{O}$ using inverse transform sampling (Line 11 in Algorithm \ref{alg:meta}).
Figure \ref{fig:algo_conv} shows qualitative illustration of Algorithm \ref{alg:meta}, for $\beta=0.001$ and $\eps=0.1$ (typical values).

\subsection{Training}\label{s:training}
Our system consists of two Multi-Layer Perceptrons (MLP): (i) $\vf_\varphi$ approximating the SDF of the learned geometry, as well as global geometry feature $\vz$ of dimension $256$, \ie, $\vf_\varphi(\vx) = (d(\vx),\vz(\vx))\in\Real^{1+256}$, where $\varphi$ denotes its learnable parameters; (ii) $L_\psi(\vx,\vn,\vv,\vz)\in \Real^3$ representing the scene's radiance field with learnable parameters $\psi$. In addition we have two scalar learnable parameters $\alpha,\beta\in \Real$. In fact, in our implementation we make the choice $\alpha=\beta^{-1}$. We denote by $\theta \in \Real^p$ the collection of all learnable parameters of the model, $\theta=(\varphi,\psi,\beta)$. 
To facilitate the learning of high frequency details of the geometry and radiance field, we exploit positional encoding \citep{mildenhall2020nerf} for the position $\vx$ and view direction $\vv$ in the geometry and radiance field. The influence of different positional encoding choices are presented in the supplementary.

Our data consists of a collection of images with camera parameters. From this data we extract pixel level data: for each pixel $p$ we have a triplet $(I_p,\vc_p,\vv_p)$, where $I_p\in\Real^3$ is its intensity (RGB color), $\vc_p\in\Real^3$ is its camera location, and $\vv_p\in\Real^3$ is the viewing direction (camera to pixel). 
Our training loss consists of two terms:
\begin{equation}\label{e:loss}
\gL(\theta) = \gL_{\text{RGB}}(\theta) + \lambda \gL_{\text{SDF}}(\varphi),\quad \text{where }\vspace{-5pt}
\end{equation}
\begin{equation}\label{e:losses}
    \gL_{\text{RGB}}(\theta)=\E_{p} \norm{I_p - \hat{I}_{\gS}(\vc_p,\vv_p)}_1,\ \ \ \text{and } \ \gL_{\text{SDF}}(\varphi) = \E_\vz \parr{\norm{\nabla d(\vz)}-1}^2,
\end{equation}
where $\gL_{\text{RGB}}$ is the color loss; $\norm{\cdot}_1$ denotes the $1$-norm, $\gS$ is computed with Algorithm \ref{alg:meta}, and $\hat{I}_\gS$ is the numerical approximation to the volume rendering integral in \eqref{e:numerical_int}; here we also incorporate the global feature in the radiance field, \ie, $L_i = L_\psi(\vx(s_i), \vn(s_i), \vv_p, \vz(\vx(s_i)))$. $\gL_{\text{SDF}}$ is the Eikonal loss encouraging $d$ to approximate a signed distance function \citep{gropp2020implicit}; the samples $\vz$ are taken to combine a single random uniform space point and a single point from $\gS$ for each pixel $p$. We train with batches of size 1024 pixels $p$. $\lambda$ is a hyper-parameter set to $0.1$ throughout the the experiments. Further implementation details are provided in the supplementary. \vspace{-5pt}


\begin{figure}[t]\hspace{-12pt}
    \includegraphics[width=\textwidth]{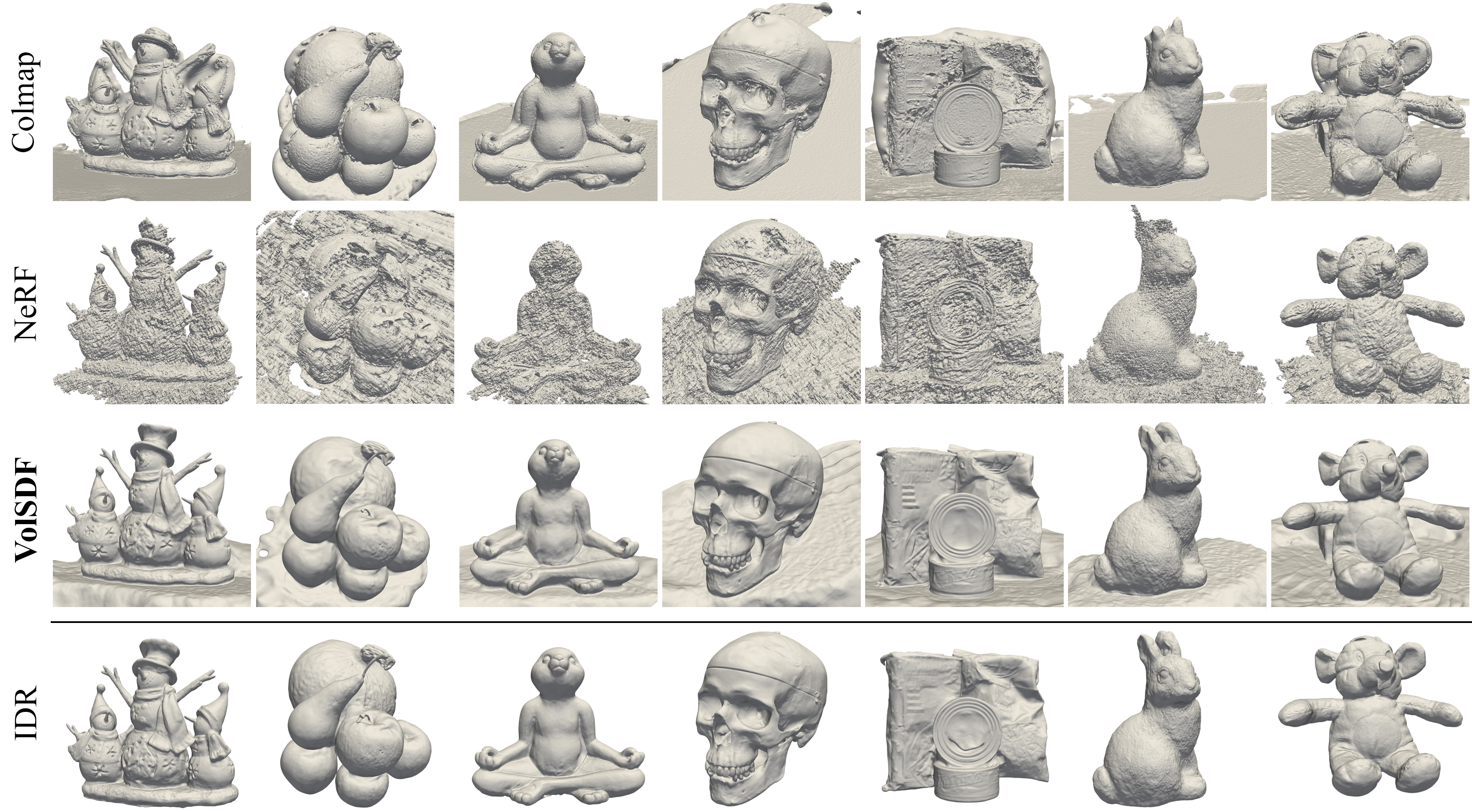}
    \caption{Qualitative results for reconstructed geometries of objects from the DTU dataset.} \label{fig:dtu_res}
\end{figure}
\begin{table}[h]
    \hspace{-8pt}
    \centering
    \small
       \begin{adjustbox}{width=\linewidth,center}
        \centering
        \renewcommand{\arraystretch}{1.5}
        \begin{tabular}[t]{@{\hskip0.0pt}c@{\hskip5.0pt}l@{\hskip6.0pt}c@{\hskip5.0pt}c@{\hskip5.0pt}c@{\hskip5.0pt}c@{\hskip5.0pt}c@{\hskip5.0pt}c@{\hskip5.0pt}c@{\hskip5.0pt}c@{\hskip5.0pt}c@{\hskip5.0pt}c@{\hskip5.0pt}c@{\hskip5.0pt}c@{\hskip5.0pt}c@{\hskip5.0pt}c@{\hskip5.0pt}c@{\hskip5.0pt}c@{\hskip0.0pt}}
            \toprule
            & Scan & 24 & 37 & 40 & 55 & 63 & 65 & 69 & 83 & 97 & 105 & 106 & 110 & 114 & 118 & 122 & \textbf{Mean}\\
         \toprule
         \multirow{5}{*}{\rotatebox[origin=c]{90}{Chamfer Distance}}
         & \textbf{IDR} & $1.63$ & $1.87$ & $0.63$ & $0.48$ & $1.04$ & $0.79$ & $0.77$ & $1.33$ & $1.16$ & $0.76$ & $0.67$ & $0.90$ & $0.42$ & $0.51$ & $0.53$ & $0.90$ \\ 
         & \textbf{colmap$_{7}$} & $0.45$ & $0.91$ & $0.37$ & $0.37$ & $0.90$ & $1.00$ & $0.54$ & $1.22$ & $1.08$ & $0.64$ & $0.48$ & $0.59$ & $0.32$ & $0.45$ & $0.43$ & $0.65$ \\
         \cline{2-18}
         & \textbf{colmap$_{0}$} & $\mathbf{0.81}$ & $2.05$ & $\mathbf{0.73}$ & $1.22$ & $1.79$ & $1.58$ & $1.02$ & $3.05$ & $1.40$ & $2.05$ & $1.00$ & $1.32$ & $0.49$ & $0.78$ & $1.17$ & $1.36$ \\
         & \textbf{NeRF} & $1.92$ & $1.73$ & $1.92$ & $0.80$ & $3.41$ & $1.39$ & $1.51$ & $5.44$ & $2.04$ & $1.10$ & $1.01$ & $2.88$ & $0.91$ & $1.00$ & $0.79$ & $1.89$ \\
         & \textbf{VolSDF} & $1.14$ & $\mathbf{1.26}$ & $0.81$ & $\mathbf{0.49}$ & $\mathbf{1.25}$ & $\mathbf{0.70}$ & $\mathbf{0.72}$ & $\mathbf{1.29}$ & $\mathbf{1.18}$ & $\mathbf{0.70}$ & $\mathbf{0.66}$ & $\mathbf{1.08}$ & $\mathbf{0.42}$ & $\mathbf{0.61}$ & $\mathbf{0.55}$ & $\mathbf{0.86}$  \\
         \midrule
         \parbox[t]{3mm}{\multirow{2}{*}{\rotatebox[origin=c]{90}{PSNR}}} 
         & \textbf{NeRF} & $26.24$ & $25.74$ & $26.79$ & $27.57$ & $31.96$ & $31.50$ & $29.58$ & $32.78$ & $28.35$ & $32.08$ & $33.49$ & $31.54$ & $31.0$ & $35.59$ & $35.51$ & $30.65$ \\
         & \textbf{VolSDF} & $26.28$ & $25.61$ & $26.55$ & $26.76$ & $31.57$ & $31.5$ & $29.38$ & $33.23$ & $28.03$ & $32.13$ & $33.16$ & $31.49$ & $30.33$ & $34.9$ & $34.75$ & $30.38$ \\
         
        \bottomrule
              \end{tabular}
        \end{adjustbox}
      
    \caption{
     Quantitative results for the DTU dataset. \vspace{-17pt}}
    \label{tab:multiview_reconstruction}
\end{table}

\section{Experiments}\vspace{-5pt}
We evaluate our method on the challenging task of multiview 3D surface reconstruction. We use two datasets: DTU \citep{jensen2014large} and BlendedMVS \citep{yao2020blendedmvs}, both containing real objects with different materials that are captured from multiple views. In Section \ref{s:recon} we show qualitative and quantitative 3D surface reconstruction results of VolSDF, comparing favorably to relevant baselines. In Section \ref{ss:dis} we demonstrate that, in contrast to NeRF \citep{mildenhall2020nerf}, our model is able to successfully disentangle the  geometry and appearance of the captured objects.

\subsection{Multi-view 3D reconstruction}\label{s:recon}
\textbf{DTU}
The DTU \citep{jensen2014large} dataset contains multi-view image ($49$ or $64$) of different objects with fixed camera and lighting parameters. We evaluate our method on the $15$ scans that were selected by \cite{yariv2020multiview}. We compare our surface accuracy using the Chamfer $l_1$ loss (measured in mm) to COLMAP$_0$ (which is watertight reconstruction; COLMAP$_7$ is not watertight and provided only for reference) \citep{schoenberger2016mvs}, NeRF~\citep{mildenhall2020nerf} and IDR~\citep{yariv2020multiview}, where for fair comparison with IDR we only evaluate the reconstruction inside the visual hull of the objects (defined by the segmentation masks of \cite{yariv2020multiview}). 
We further evaluate the PSNR of our rendering compared to \cite{mildenhall2020nerf}. 
Quantitative results are presented in Table \ref{tab:multiview_reconstruction}. It can be observed that our method is on par with IDR (that uses object masks for all images) and outperforms NeRF and COLMAP in terms of reconstruction accuracy. Our rendering quality is comparable to NeRF's.

\begin{figure}\hspace{-12pt}
    \includegraphics[width=\textwidth]{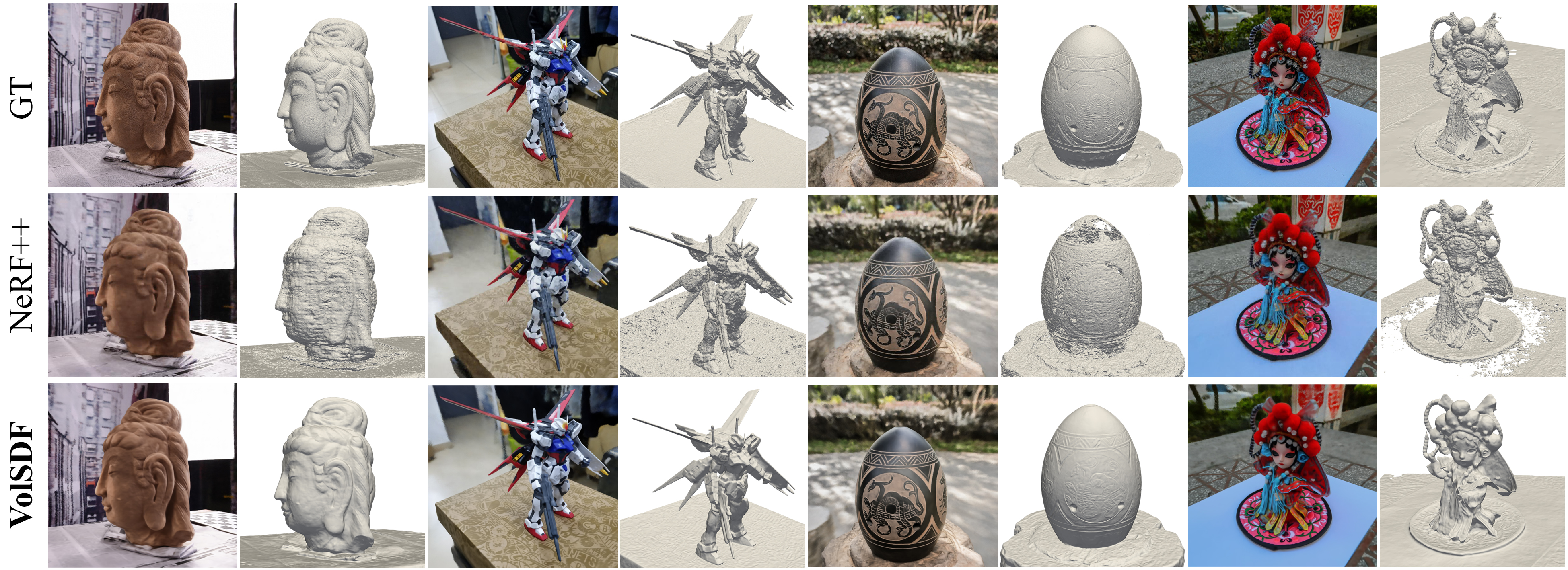}
    \caption{Qualitative results sampled from the BlendedMVS dataset. 
    For each scan we present a visualization of a rendered image and the reconstructed 3D geometry. \vspace{-10pt}} \label{fig:blended_res}
\end{figure}

\begin{table}[t]
    \hspace{-8pt}
    \setlength\tabcolsep{6pt} 

        \begin{adjustbox}{width=\textwidth,center}
        \aboverulesep=0ex
        \belowrulesep=0ex
        \renewcommand{\arraystretch}{1.5}
        \begin{tabular}[t]{@{\hskip0.0pt}llcccccccccc@{\hskip0.0pt}}
            \toprule
            & Scene & Doll & Egg & Head & Angel & Bull & Robot & Dog & Bread & Camera & \textbf{Mean}\\
         \toprule
        
        \multirow{1}{*}{Chamfer $l_1$} & 
        Our Improvement ($\%$)
        & $54.0$ & $91.2$ & $24.3$ & $75.1$ & $60.7$ & $27.2$ & $47.7$ & $34.6$ & $51.8$ & $51.8$ \\ 
         \cline{2-12}
        
         \midrule
         \multirow{2}{*}{PSNR}
         & \textbf{NeRF++} & $26.95$ & $27.34$ & $27.23$ & $30.06$ & $26.65$ & $26.73$ & $27.90$ & $31.68$ & $23.44$ & $27.55$  \\
         & \textbf{VolSDF} & $25.49$ & $27.18$ & $26.36$ & $29.79$ & $26.01$ & $26.03$ & $28.65$ & $31.24$ & $22.97$ & $27.08$  \\
         
        \bottomrule
              \end{tabular}
        \end{adjustbox}
        \vspace{3pt}
      
    \caption{
     Quantitative results for the BlendedMVS dataset. 
     \vspace{-17pt}}
    \label{tab:blended}
\end{table}

\textbf{BlendedMVS} The BlendedMVS dataset \citep{yao2020blendedmvs} contains a large collection of $113$ scenes captured from multiple views. It supplies high quality ground truth 3D models for evaluation, various camera configurations, and a variety of indoor/outdoor real environments. We selected $9$ different scenes and used our method to reconstruct the surface of each object. In contrast to the DTU dataset, BlendedMVS scenes have complex backgrounds. Therefore we use NeRF++ \cite{kaizhang2020} as a baseline for this dataset. In Table~\ref{tab:blended} we present our results compared to NeRF++. Qualitative comparisons are presented in Fig.~\ref{fig:blended_res}; since the units are unknown in this case we present relative improvement of 
\begin{wrapfigure}[9]{r}{0.31\textwidth}
   \centering
    \vspace{-3pt}
    \includegraphics[width=0.3\textwidth]{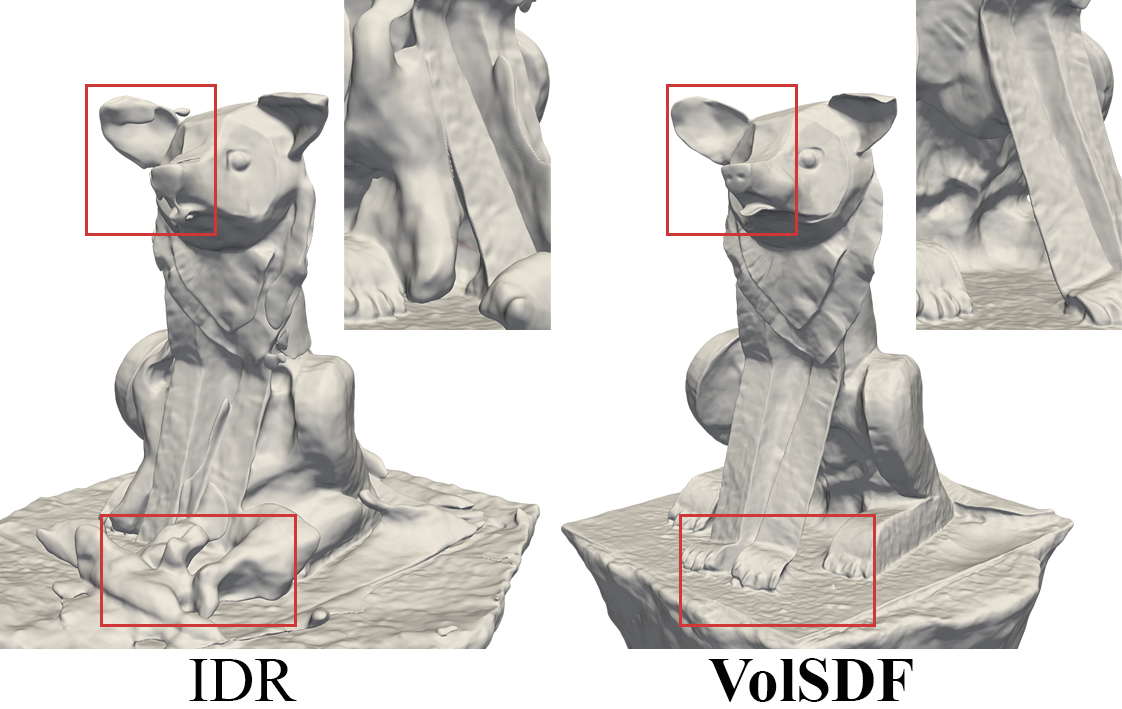}
 \caption{IDR extraneous parts.}
\label{fig:ba_gt}
\end{wrapfigure}
Chamfer distance (in $\%$) compared to NeRF. Also in this case, we improve NeRF reconstructions considerably, while being on-par in terms of the rendering quality (PSNR).

\textbf{Comparison to \cite{yariv2020multiview}}
 IDR \cite{yariv2020multiview} is the state of the art 3D surface reconstruction method using implicit representation. However, it suffers from two drawbacks: first, it requires object masks for training, which is a strong supervision signal. 
 Second, since it sets the pixel color based only on the single point of intersection of the corresponding viewing ray, it is more pruned to local minima that sometimes appear in the form of extraneous surface parts. Figure \ref{fig:ba_gt}  compares the same scene trained with IDR with the addition of ground truth masks, and VolSDF trained without masks. Note that IDR introduces some extraneous surface parts (\eg, in marked red), while VolSDF provides a more faithful result in this case.

\subsection{Disentanglement of geometry and appearance}\label{ss:dis}
We have tested the disentanglement of scenes to geometry (density) and appearance (radiance field) by switching the radiance fields of two trained scenes. 
For VolSDF we switched $L_\psi$. For NeRF \cite{mildenhall2020nerf} we note that the radiance field is computed as $L_\psi(\vz,\vv)$, where $L_\psi$ is a fully connected network with one hidden layer (of width 128 and ReLU activation) and $\vz$ is a feature vector. We tested two versions of NeRF disentanglement: First, by switching the original radiance fields $L_\psi$ of trained NeRF networks. Second, by switching the radiance fields of trained NeRF models with an identical radiance field model to ours, namely $L_\psi(\vx,\vn,\vv,\vz)$. 
As shown in Figure \ref{fig:dis} both versions of NeRF fail to produce a correct disentanglement in these scenes, while VolSDF successfully switches the materials of the two objects. We attribute this to the specific inductive bias injected with the use of the density in \eqref{e:density}.

\begin{figure}[h]
\centering
    \begin{tabular}{ccc}
    \includegraphics[width=0.30\textwidth]{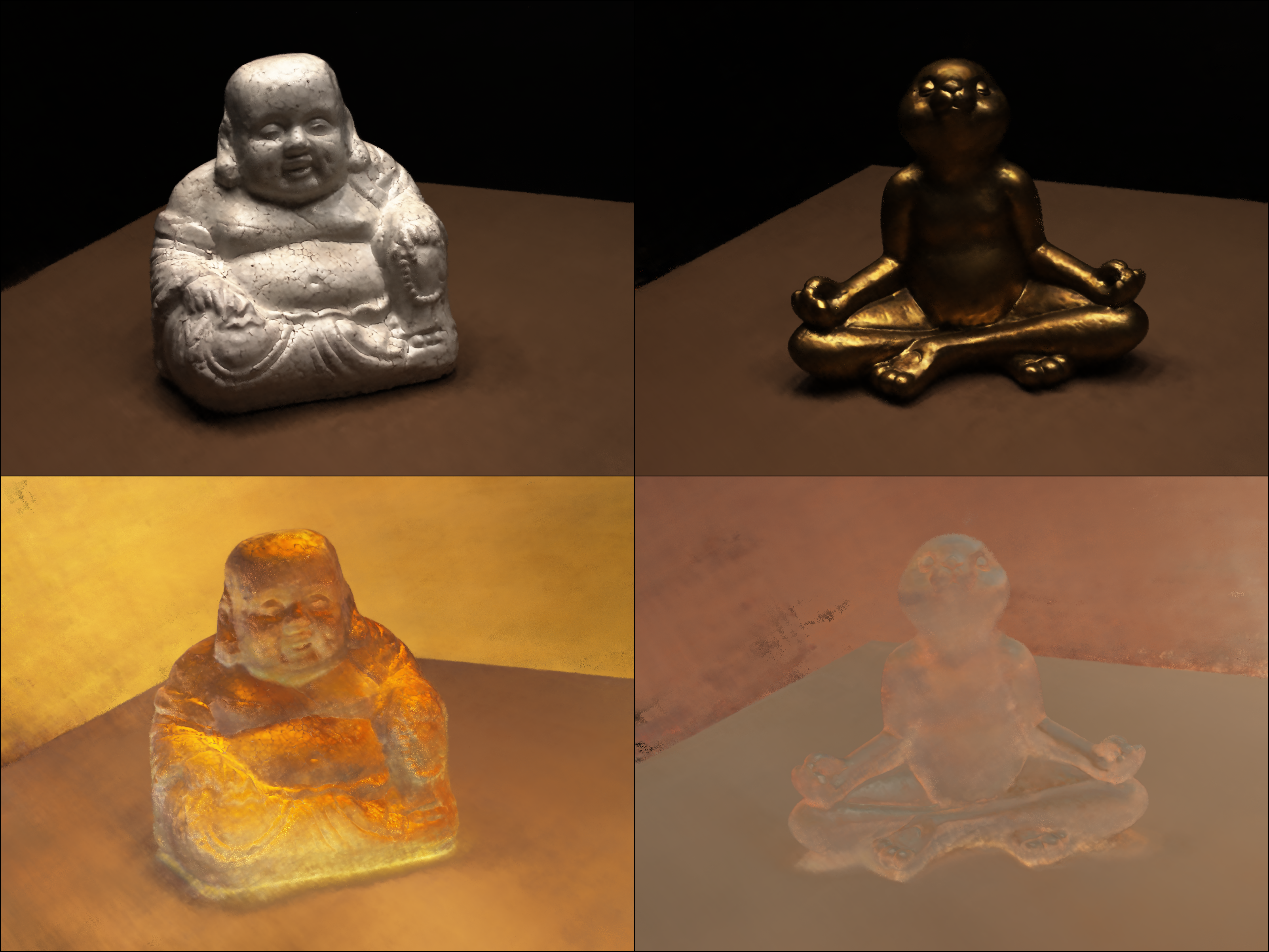} &
    \includegraphics[width=0.30\textwidth]{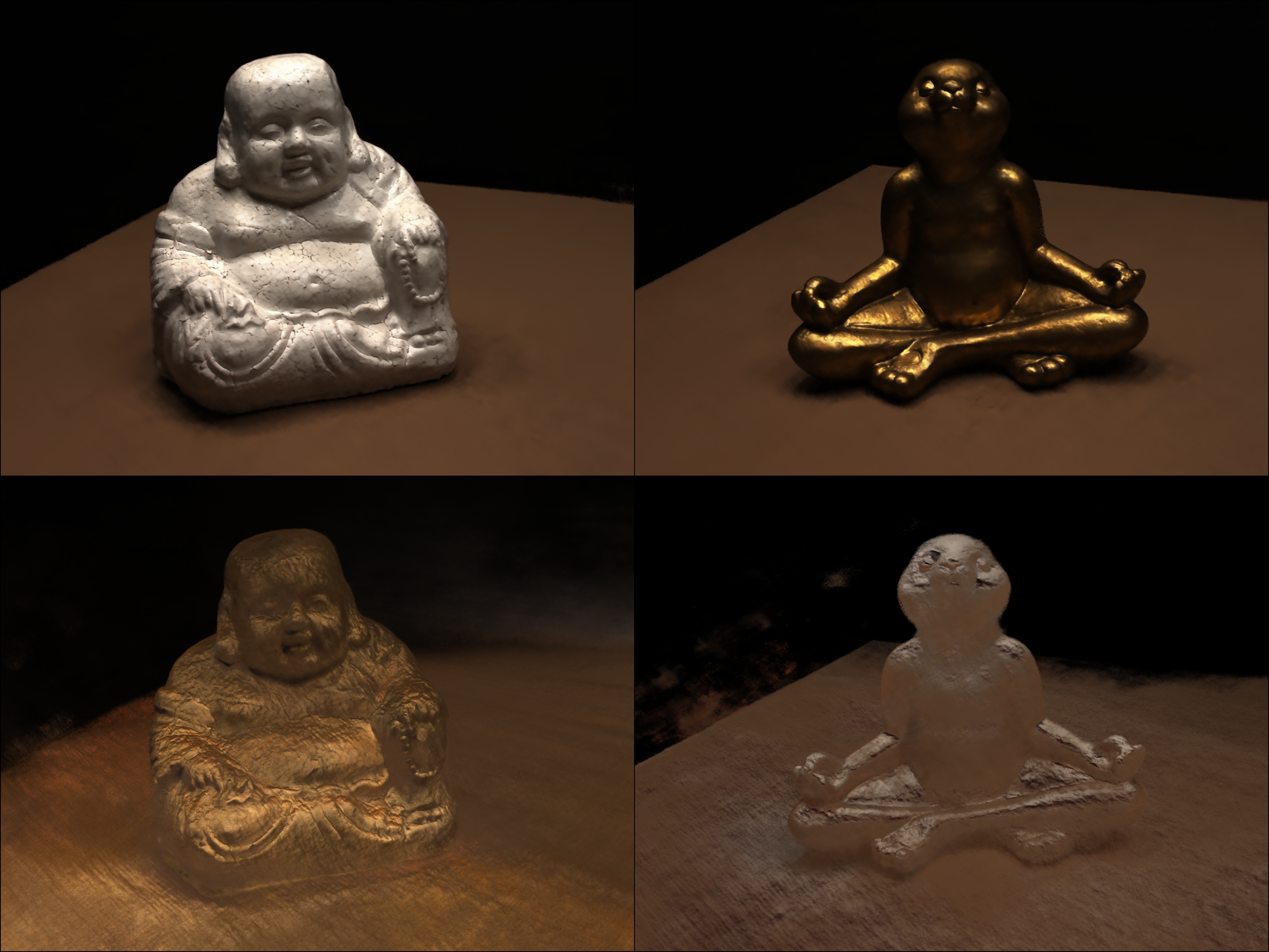} &
     \includegraphics[width=0.30\textwidth]{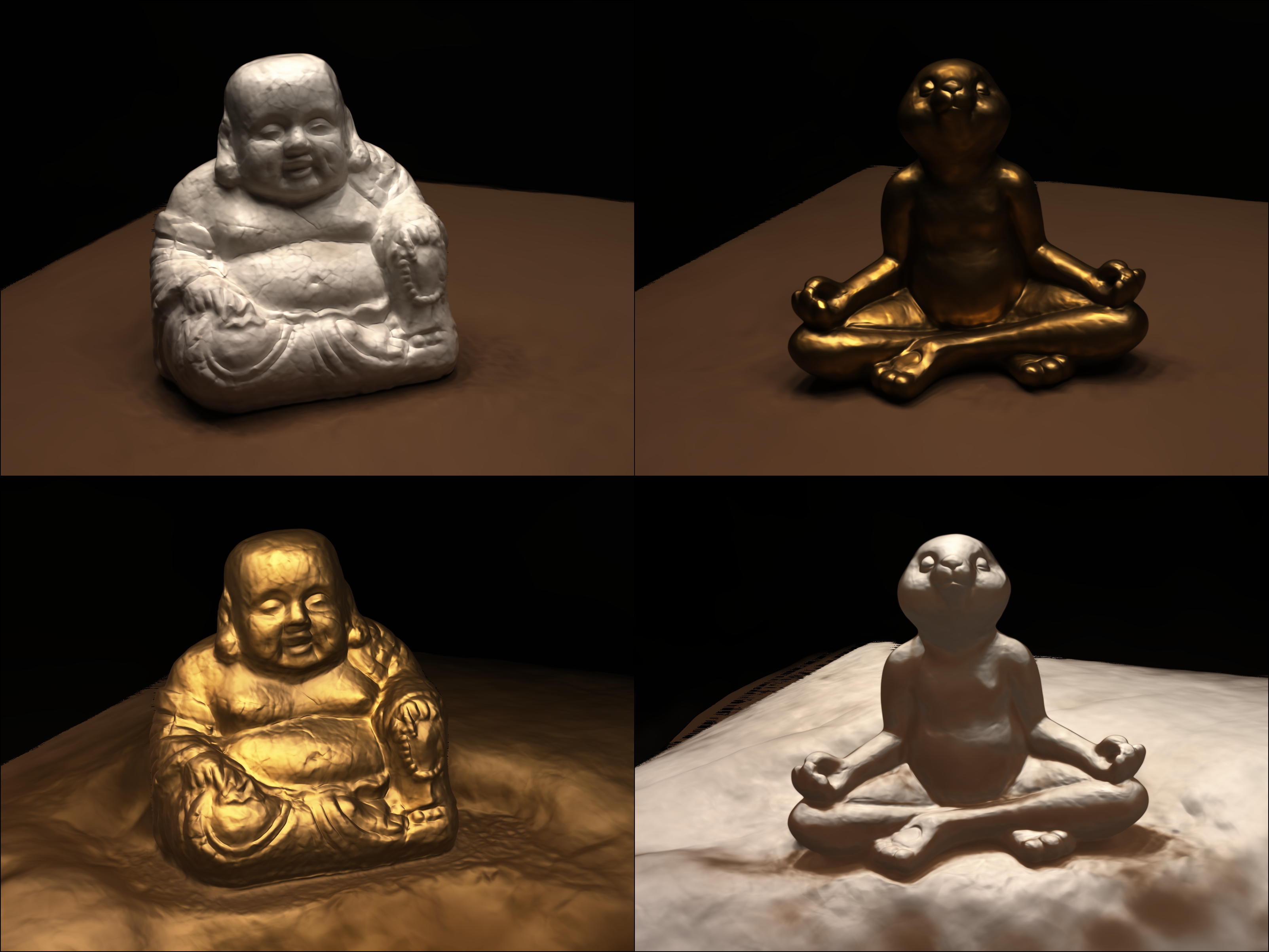} \\
    NeRF & NeRF \textit{with normal} & \textbf{VolSDF}   \vspace{-5pt} \end{tabular}
    \caption{Geometry and radiance disentanglement is physically plausible with VolSDF.\vspace{-10pt}} \label{fig:dis}
\end{figure}

\section{Conclusions}\vspace{-5pt}
We introduce VolSDF, a volume rendering framework for implicit neural surfaces.  We represent the volume density as a transformed version of the signed distance function to the learned surface geometry. This seemingly simple definition provides a useful inductive bias, allowing disentanglement of geometry (\ie, density) and radiance field, and improves the geometry approximation over previous neural volume rendering techniques. Furthermore, it allows to bound the opacity approximation error leading to high fidelity sampling of the volume rendering integral. 

Some limitations of our method present interesting future research opportunities. First, although working well in practice, we do not have a proof of correctness for the sampling algorithm. We believe providing such a proof, or finding a version of this algorithm that has a proof would be a useful contribution. In general, we believe working with bounds in volume rendering could improve learning and disentanglement and push the field forward. Second, representing non-watertight manifolds and/or manifolds with boundaries, such as zero thickness surfaces, is not possible with an SDF. Generalizations such as multiple implicits and unsigned fields could be proven valuable. Third, our current formulation assumes homogeneous density; extending it to more general density models would allow representing a broader class of geometries. Fourth, now that high quality geometries can be learned in an unsupervised manner it will be interesting to learn dynamic geometries and shape spaces directly from collections of images. Lastly, although we don't see immediate negative societal impact of our work, we do note that accurate geometry reconstruction from images can be used for malice purposes.

\section*{Acknowledgments}
LY is supported by the European Research Council (ERC Consolidator Grant, "LiftMatch" 771136), the Israel Science Foundation (Grant No. 1830/17), and Carolito Stiftung (WAIC).
YK is supported by the U.S.- Israel Binational Science Foundation, grant number 2018680, Carolito Stiftung (WAIC), and by the Kahn foundation.


{\small
\bibliographystyle{abbrv} 

\bibliography{vol_rendering_of_implicits}
}

\appendix

\section*{\Large{Supplementary Material}}

\section{Additional results}

\subsection{Sampling ablation study} 
Figure \ref{fig:ablation} depicts an ablation study that we performed for evaluating the sampling algorithm, by replacing it with other sampling strategies. We compared with the following alternatives: Uniform stands for an uniform sampling of $256$ samples along each ray ; 2-networks denotes a hierarchical sampling using coarse and fine networks as suggested in \citep{mildenhall2020nerf}; our sampling algorithm as suggested in section \ref{s:sampling_algorith} where the maximal number of iterations is set to $1$ or $5$ (the choice in the paper). We note that using $1$ iteration resembles one level of sampling as in the hierarchical sampling of \citep{mildenhall2020nerf}.

As can be seen in Figure \ref{fig:ablation} (see also reported Chamfer and PSNR scores) alternative sampling procedures lead to lower accuracy and some artifacts in the geometry (see \eg, head top, and nose areas) and rendering (see \eg, salt and pepper noise and over-smoothed areas).  


\begin{figure}[h!]
    \includegraphics[width=\textwidth]{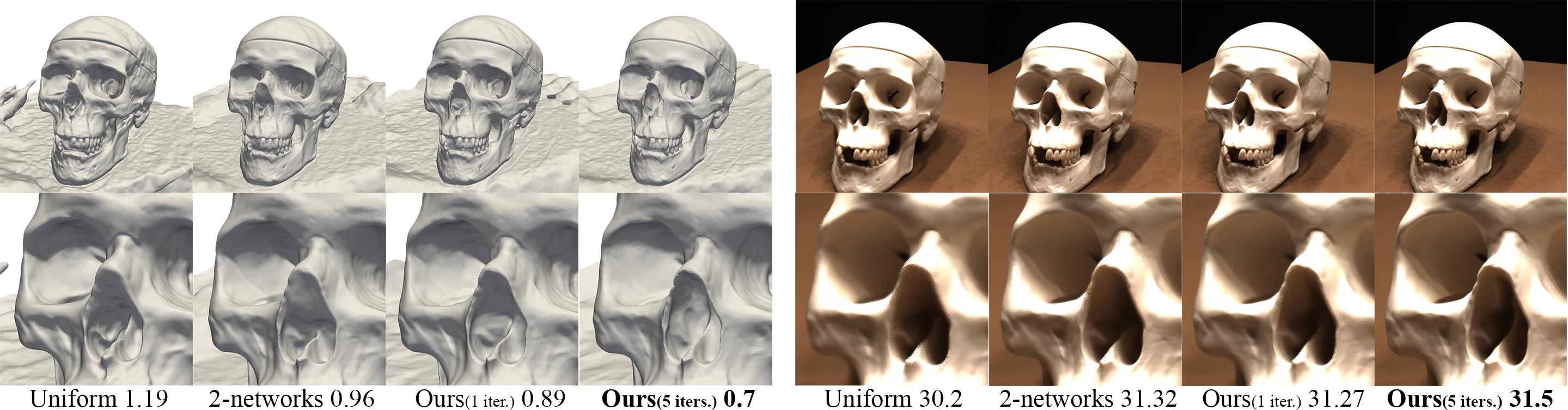}
    \caption{Sampling ablation study: the left side depicts the geometry reconstruction results with the corresponding Chamfer distances, whereas the right side presents rendering results with their corresponding PSNRs.} \label{fig:ablation}
\end{figure}

\subsection{Positional encoding ablation}\label{ss:pe_ablation}
We perform an ablation study on the level of positional encoding used in the geometry network. We note that VolSDF use level 6 positional encoding, while NeRF use level 10. In Figure \ref{fig:pe_ablation} we show the DTU Bunny scene with positional encoding levels 6 and 10 for both NeRF and VolSDF; we report both PSNR for the rendered images and Chamfer distance to the learned surfaces. Note that higher positional encoding improves specular highlights and details of VolSDF but adds some undesired noise to the reconstructed surface.

\begin{figure}[h]
    \hspace{-10pt}
    \begin{tabular}{cccccccc}
    \includegraphics[width=0.123\textwidth]{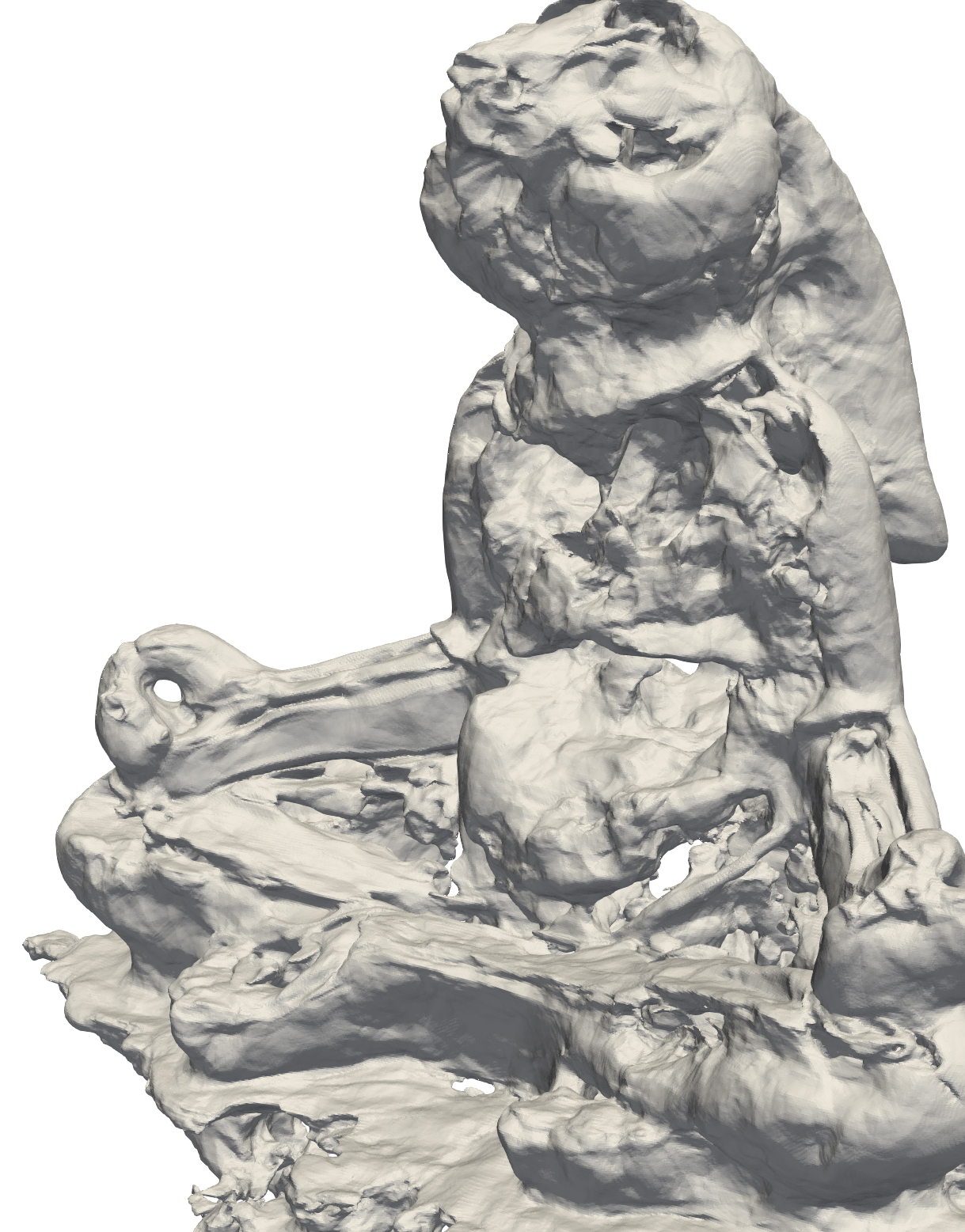} \hspace{-14pt} &
    \includegraphics[width=0.123\textwidth]{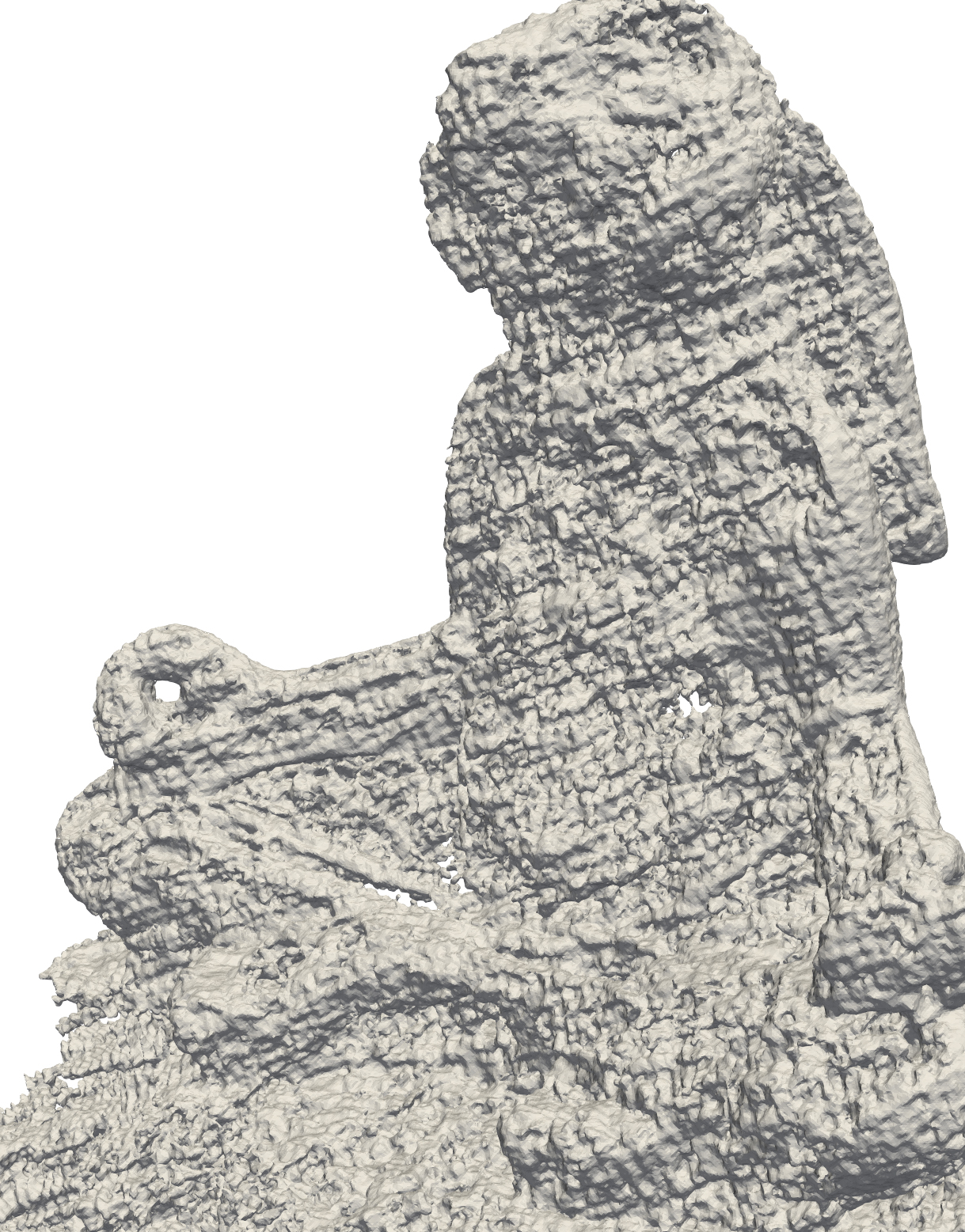} \hspace{-14pt} &
    \includegraphics[width=0.123\textwidth]{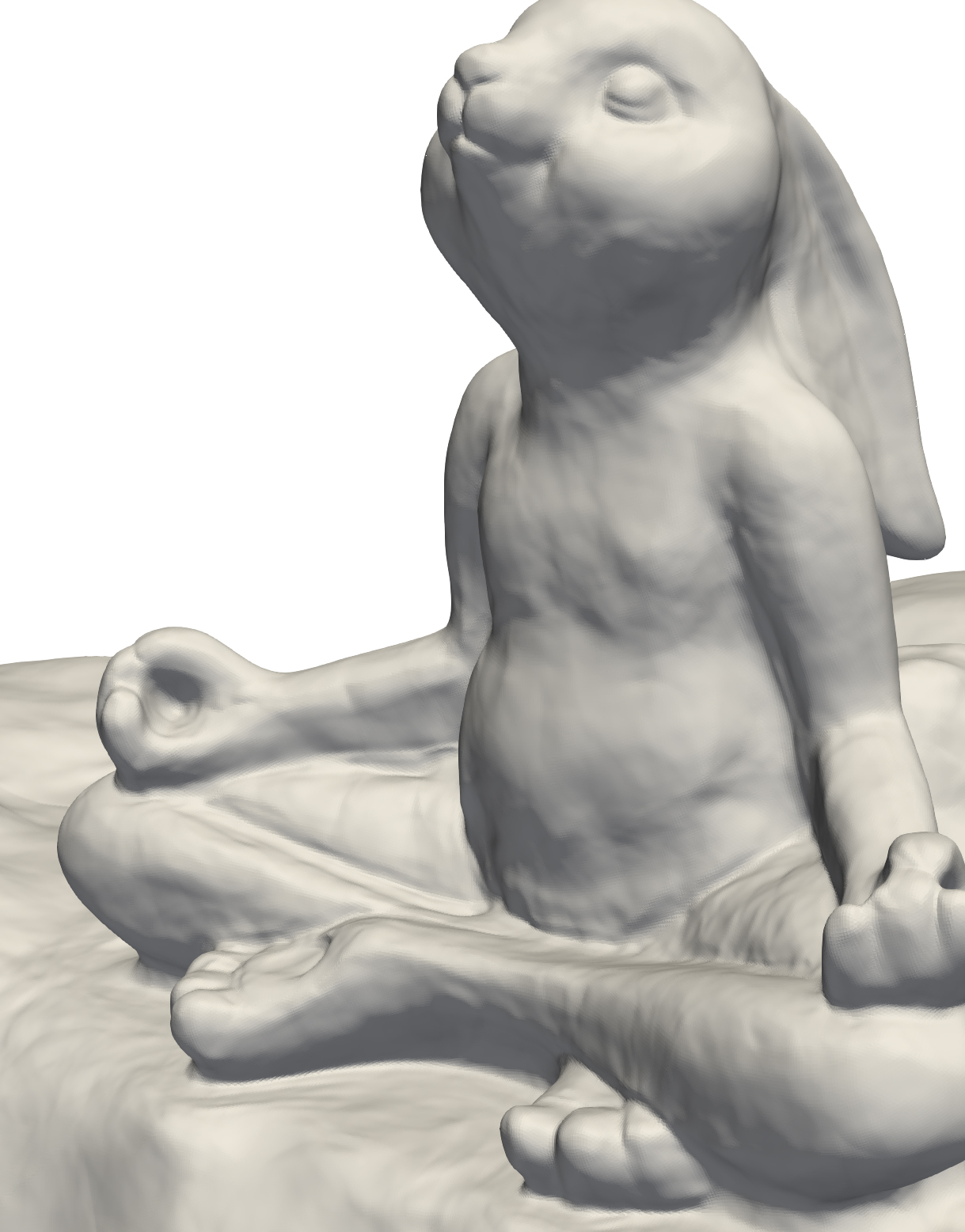} \hspace{-14pt} &
     \includegraphics[width=0.123\textwidth]{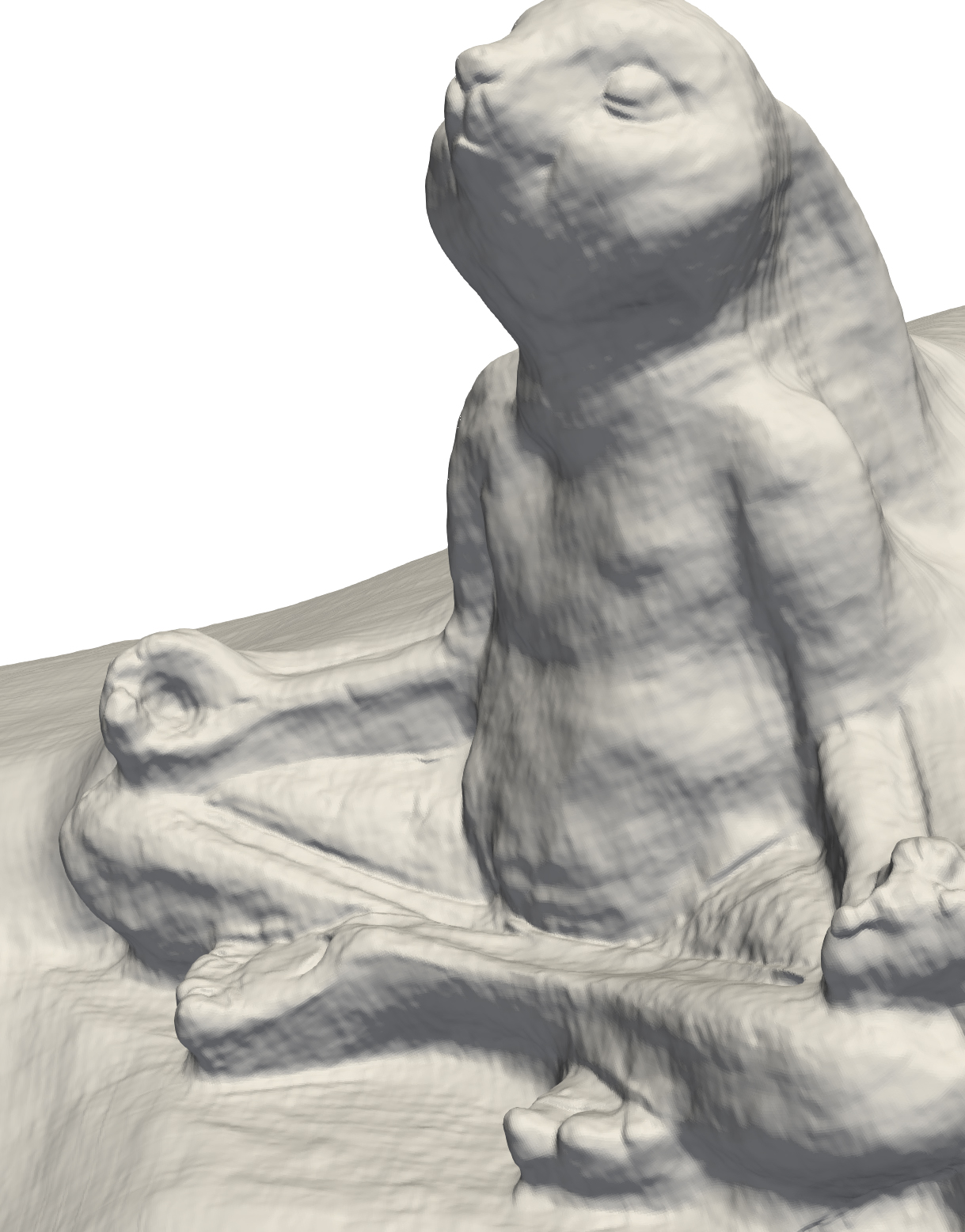}
     \hspace{-8pt} 
     &
     \includegraphics[width=0.123\textwidth]{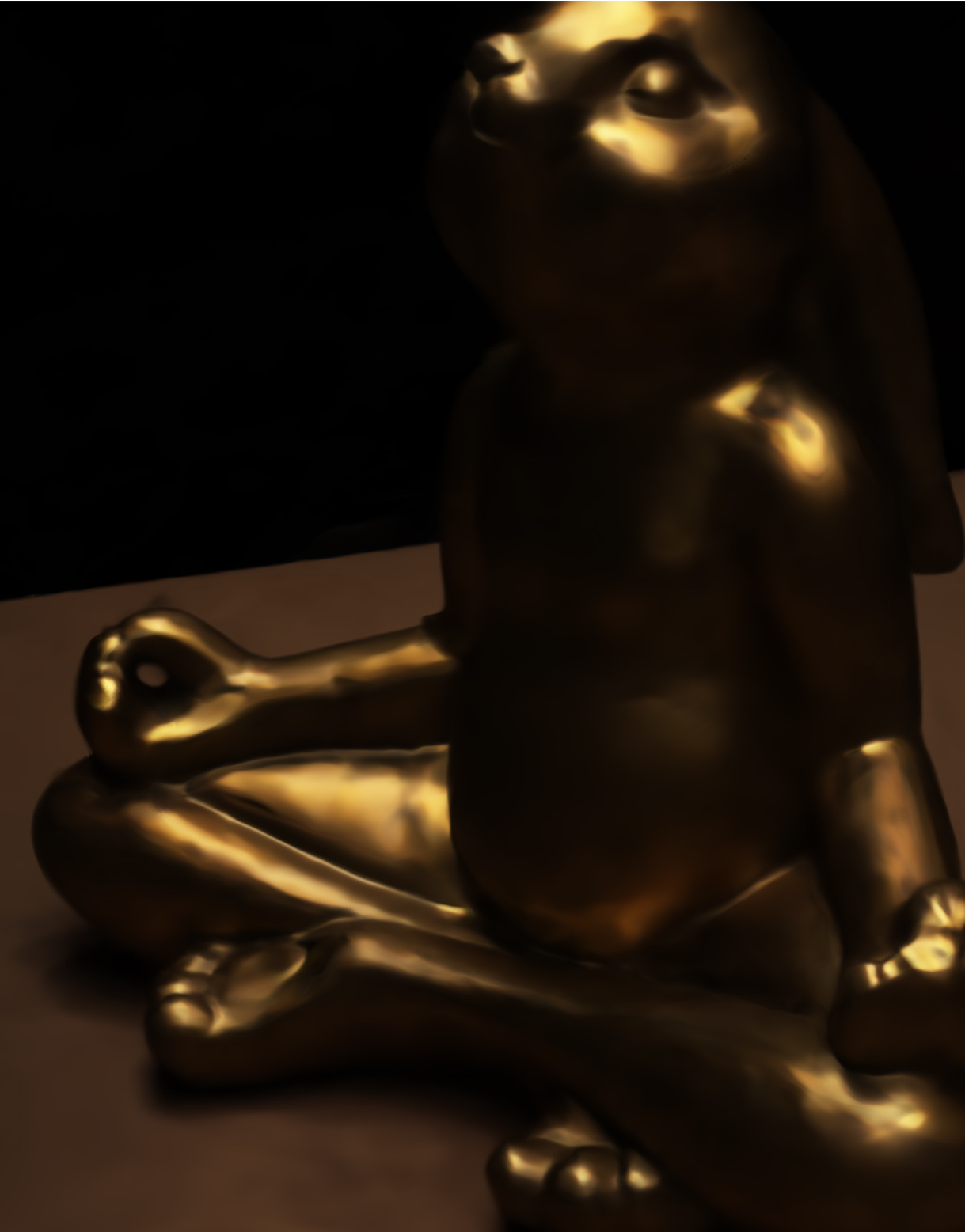} \hspace{-14pt} &
    \includegraphics[width=0.123\textwidth]{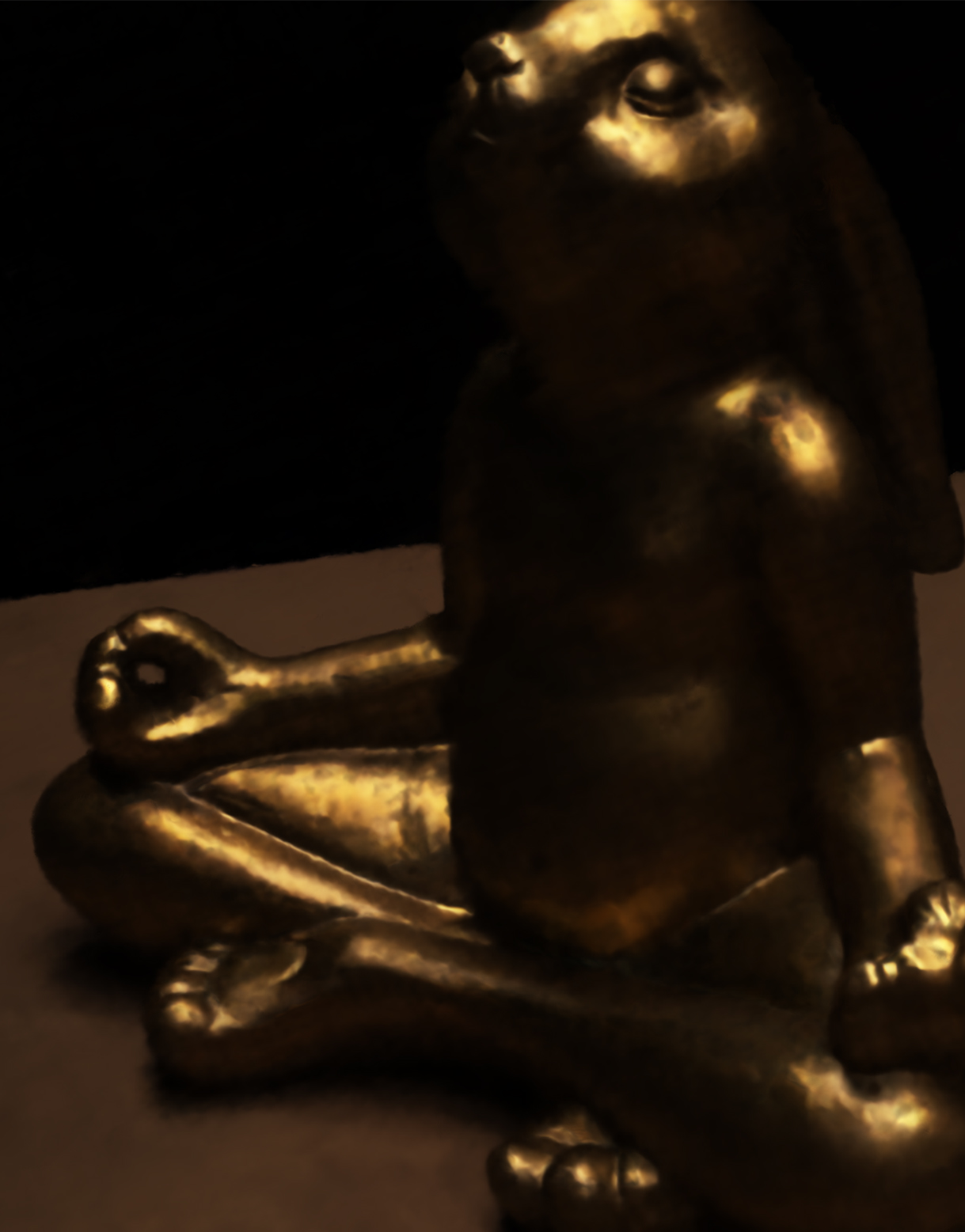} \hspace{-14pt} &
    \includegraphics[width=0.123\textwidth]{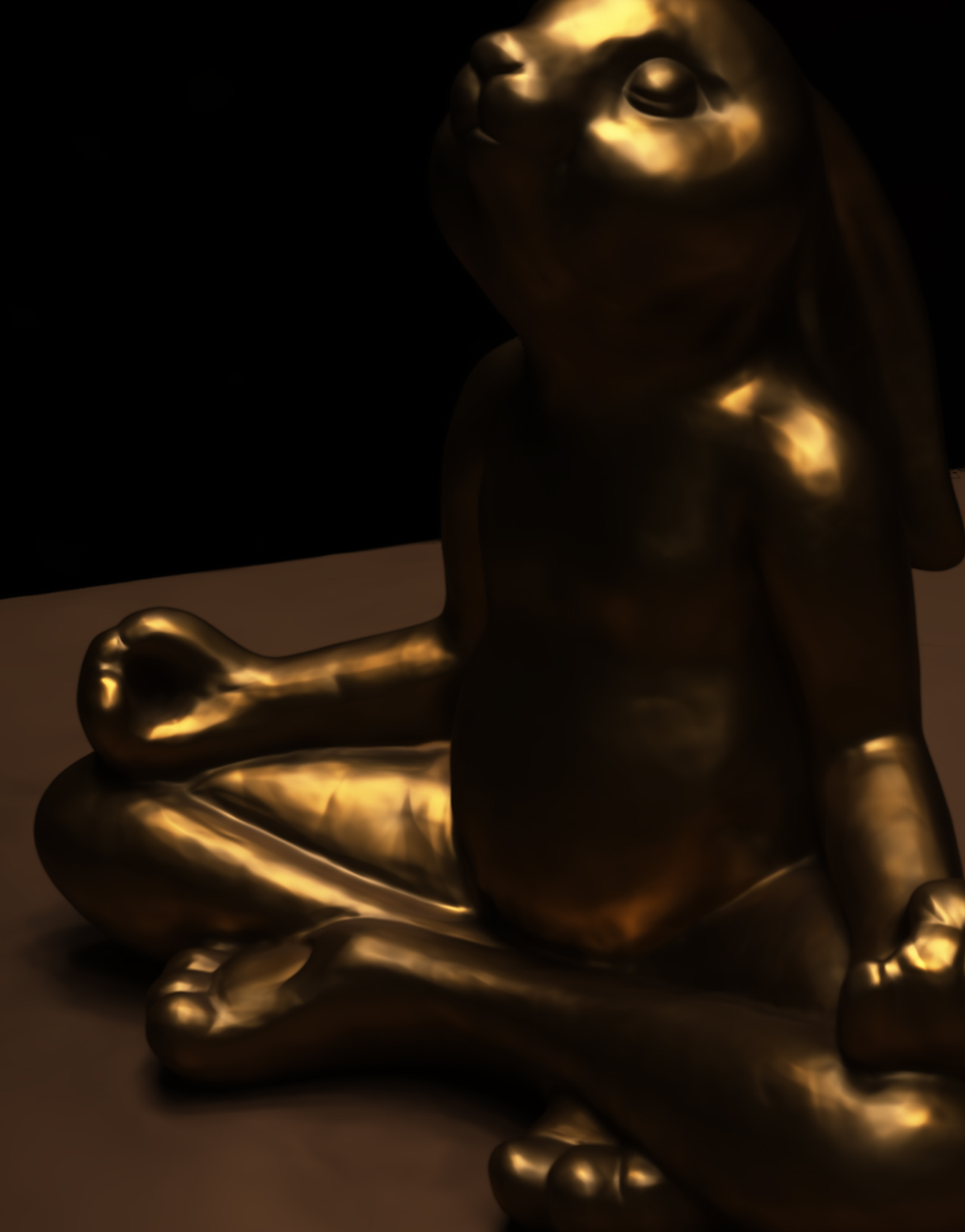} \hspace{-14pt} &
     \includegraphics[width=0.123\textwidth]{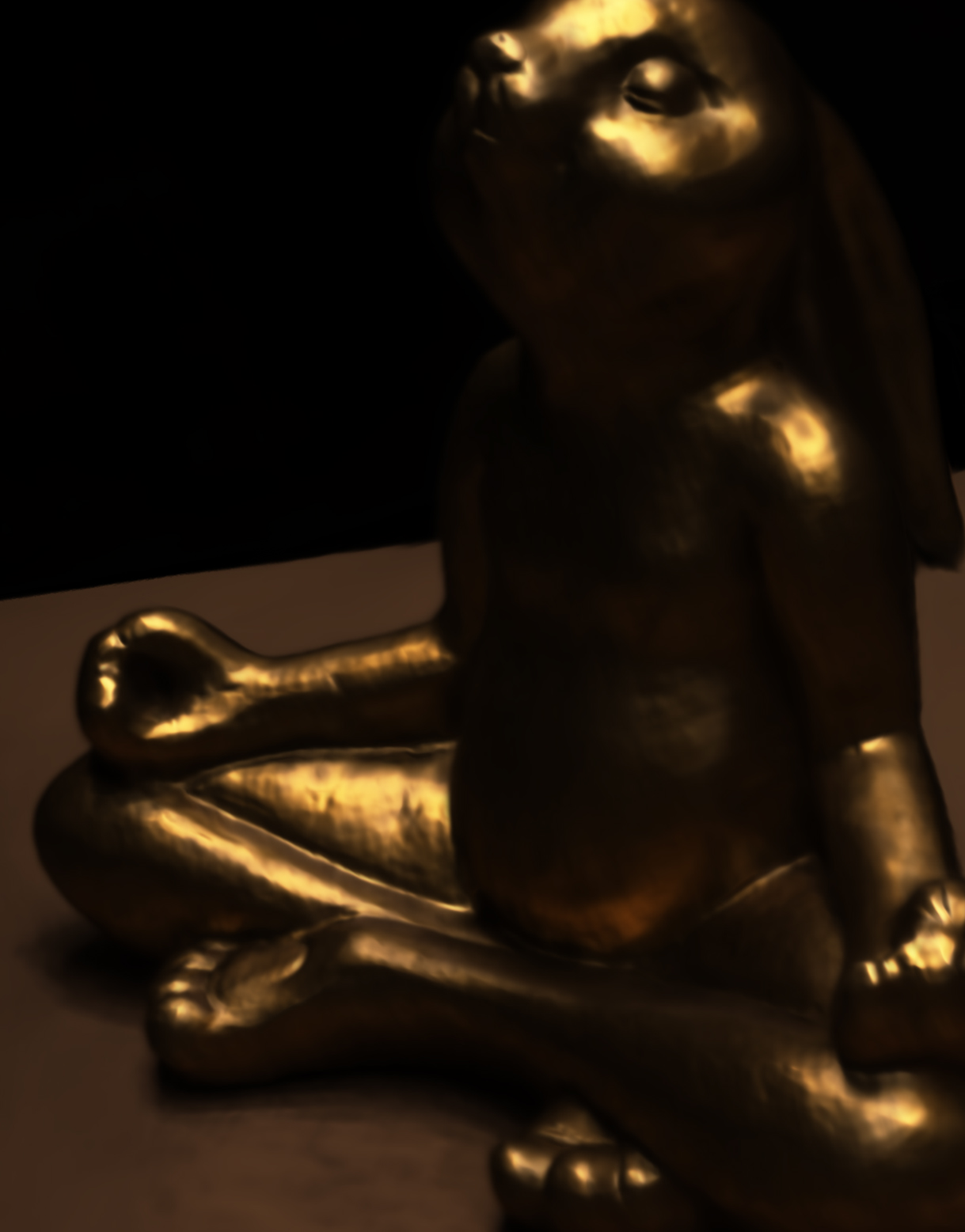}
     \vspace{-7pt} \\
    \scriptsize{NeRF - \text{PE 6}} \hspace{-14pt} &
    \scriptsize{NeRF - \text{PE 10}} \hspace{-14pt} &
    \scriptsize{\textbf{VolSDF} - \text{PE 6}} \hspace{-14pt} &
    \scriptsize{\textbf{VolSDF} - \text{PE 10}} 
    \hspace{-8pt}
    &
    \scriptsize{NeRF - \text{PE 6}} \hspace{-14pt} &
    \scriptsize{NeRF - \text{PE 10}} \hspace{-14pt} &
    \scriptsize{\textbf{VolSDF} - \text{PE 6}}\hspace{-14pt} &
    \scriptsize{\textbf{VolSDF} - \text{PE 10}}
     \vspace{-7pt} \\
    \scriptsize{3.02} \hspace{-14pt} &
    \scriptsize{2.88} \hspace{-14pt} &
    \scriptsize{1.08} \hspace{-14pt} &
    \scriptsize{2.37} 
    \hspace{-8pt}
    &
    \scriptsize{31.38} \hspace{-14pt} &
    \scriptsize{31.68} \hspace{-14pt} &
    \scriptsize{31.49}\hspace{-14pt} &
    \scriptsize{31.8}
    \end{tabular}
    \vspace{-5pt}
    \caption{{Positional Encoding (PE) ablation. We  note the tradeoff between smoother geometry with PE level 6 versus detailed rendering and slightly higher noise with PE level 10. Note that in both cases NeRF fails to decompose correctly density and radiance field.}} \label{fig:pe_ablation}
\end{figure}

\subsection{Multi-view 3D reconstruction}
Figures \ref{fig:dtu_res_supp} and \ref{fig:blended_res_supp} show additional qualitative results for DTU and BlendedMVS datasets, respectively. 
We further provide a \textbf{video} with qualitative rendering results of the learned geometries and radiance fields from simulated camera paths  (including novel views); note that VolSDF rendering alleviates NeRF's salt and pepper artifacts, while producing higher fidelity geometry approximation.

\begin{figure}[h]
    \includegraphics[width=\textwidth]{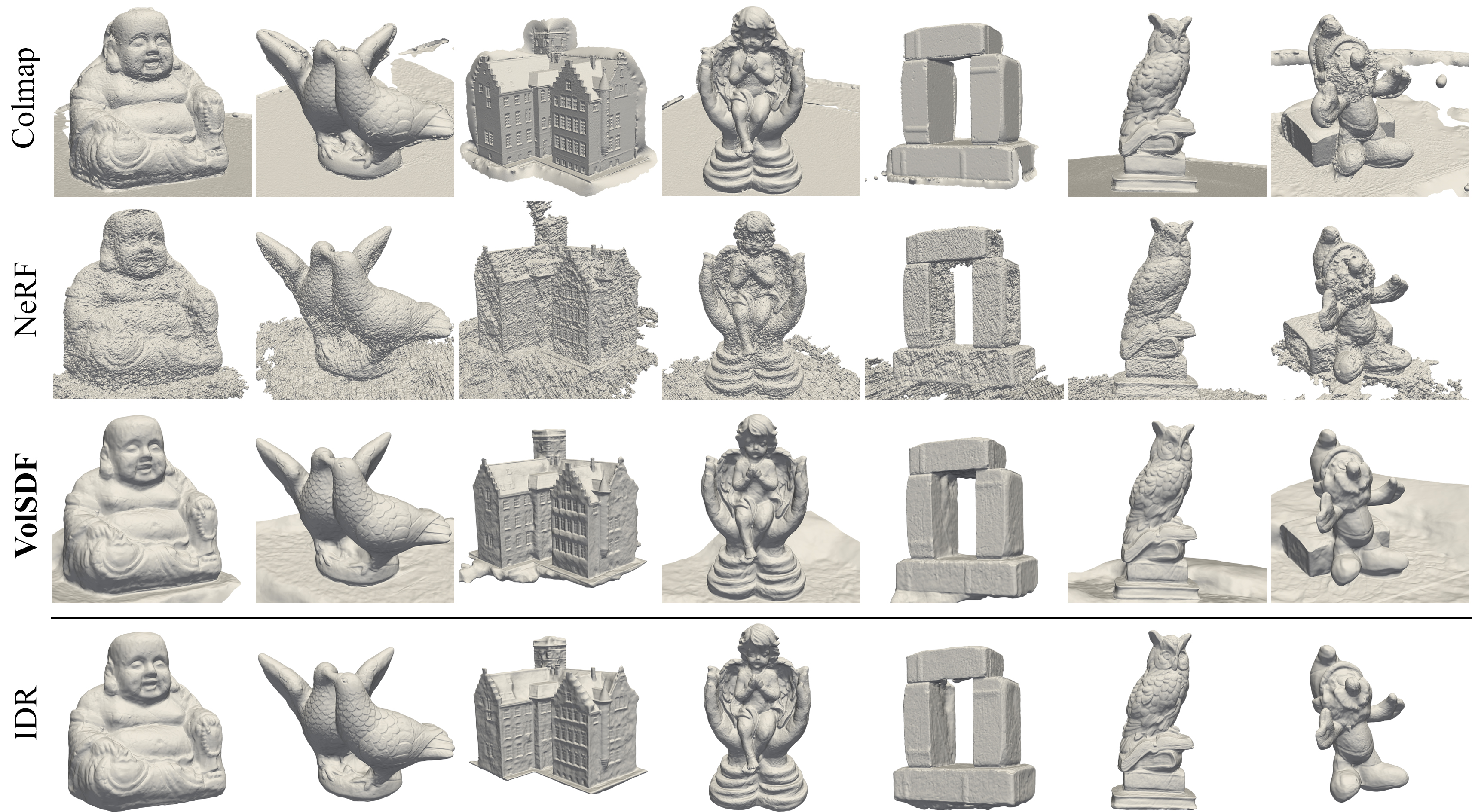}
    \caption{Qualitative results for the reconstructed geometries of objects from the DTU dataset.} \label{fig:dtu_res_supp}
\end{figure}

\begin{figure}[h]
    \includegraphics[width=\textwidth]{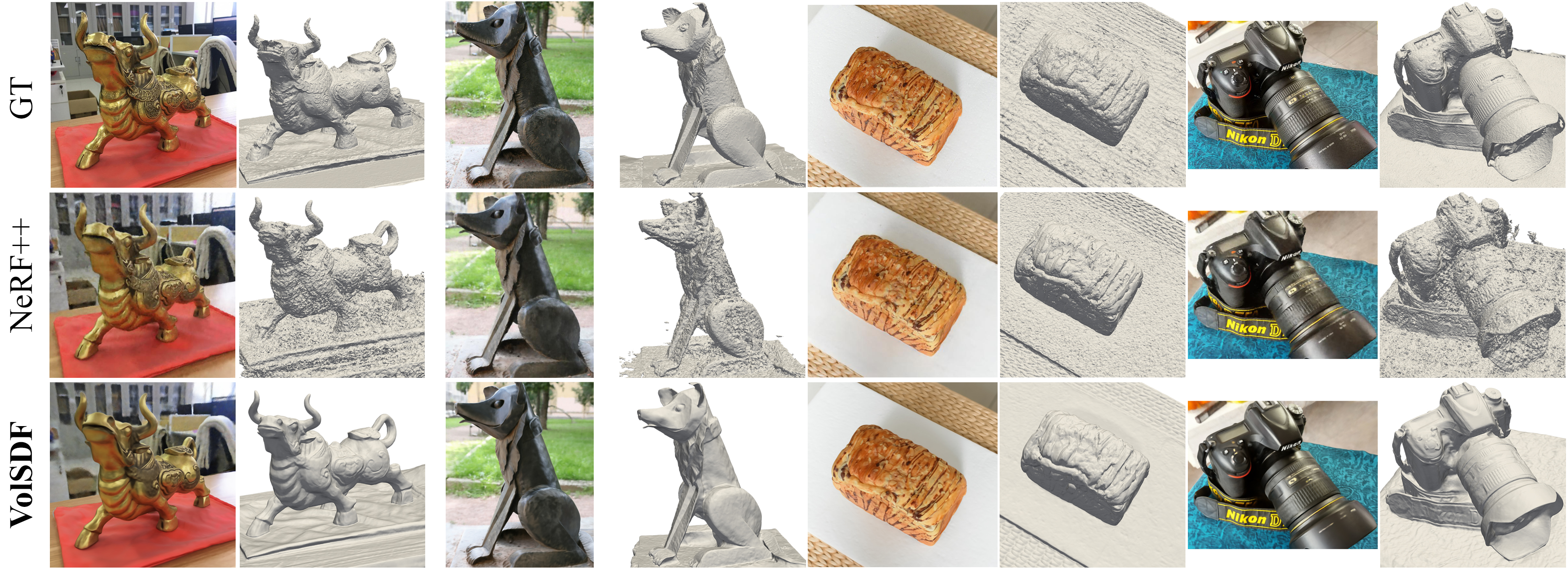}
    \caption{Qualitative results sampled from the BlendedMVS dataset. For each scan we present a visualization of a rendered image and the 3D geometry.} \label{fig:blended_res_supp}
\end{figure}

\begin{wrapfigure}[14]{r}{0.65\textwidth}
  \centering
  \vspace{-12pt}
    \includegraphics[width=0.65\textwidth]{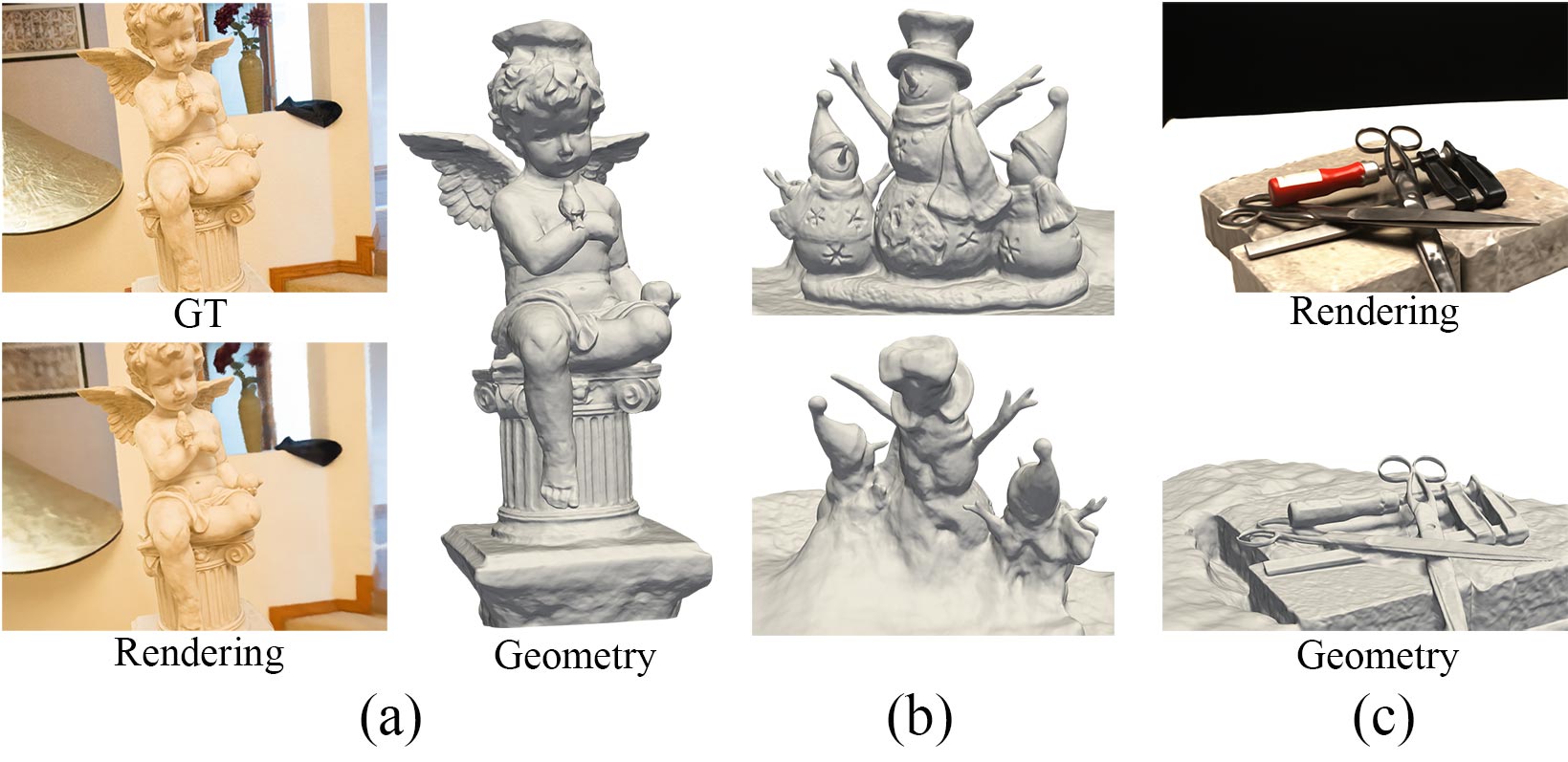}
 \caption{Failure cases, see details in the text.}
\label{fig:failure}
\end{wrapfigure}

\subsection{Limitations}
Figure \ref{fig:failure} shows the main failure cases of our method: First, we observe that the geometry for unseen regions is not well defined and can be completed arbitrarily by the algorithm, see \eg, the angle statue head top in (a), and the reverse side of the snow man statue in (b). Second, homogeneous texture-less areas are hard to reconstruct faithfully, see \eg, the white background desk in (c). These limitations can potentially be alleviated with the addition of extra assumptions such as minimal surface reconstruction \citep{lipman2021phase} and/or defining background color or predefined geometry (\eg, a plane).

\subsection{Rendering comparison}
Figure \ref{fig:rendering_compare} in the main paper shows a comparison of NeRF and VolSDF rendering for the same scene using the same random sampling strategy of the inverse opacity function. For NeRF, replacing random sampling with regularly-spaced sampling introduces different artifacts as shown in Figure \ref{fig:rand_vs_linspace}. In contrast, VolSDF produces consistent rendering results regardless of the sampling strategy.

\begin{figure}[h]
\centering
    \begin{tabular}{ccc}
    \hspace{-5pt} \includegraphics[width=0.33\textwidth]{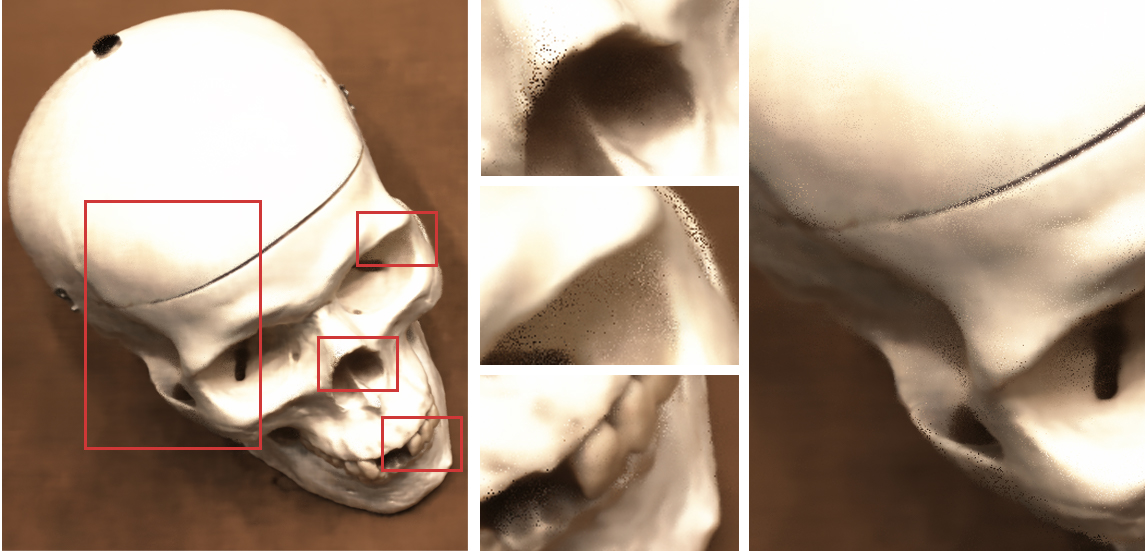} &
    \hspace{-20pt}\includegraphics[width=0.33\textwidth]{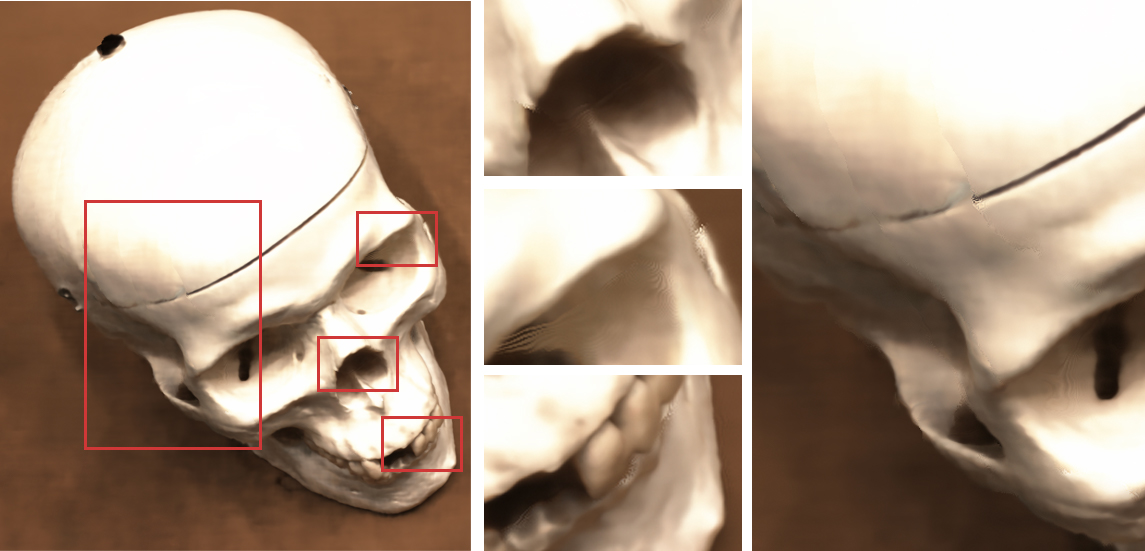} &
    \hspace{-22pt} \includegraphics[width=0.33\textwidth]{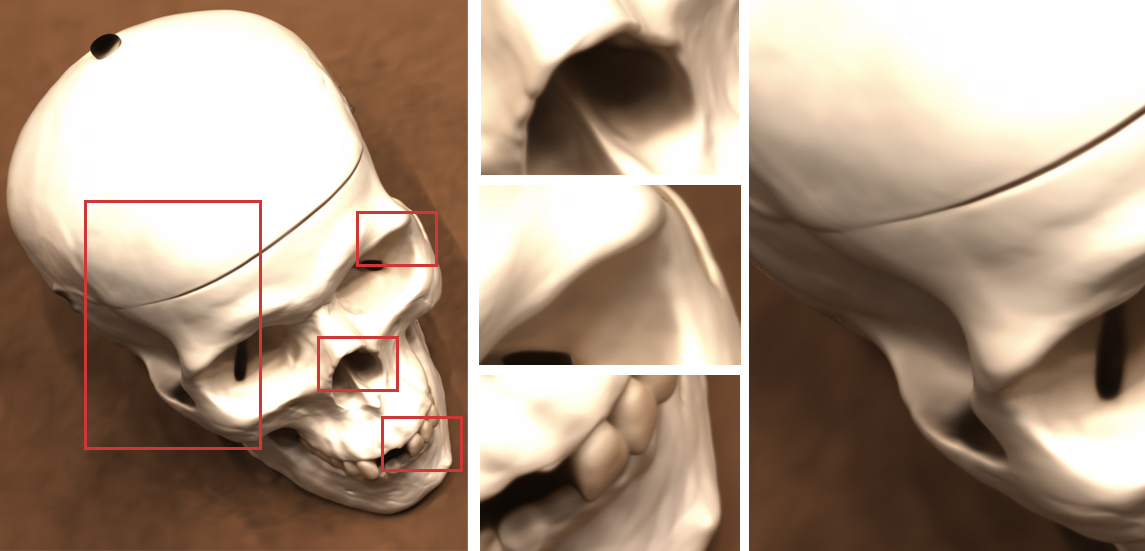} \\
    NeRF \textit{random sampling} & NeRF \textit{regularly-spaced sampling} & \textbf{VolSDF}    
    \end{tabular}
    \caption{{Comparison of random versus regularly-spaced sampling in NeRF. VolSDF produces consistent results in both cases.} } \label{fig:rand_vs_linspace}
\end{figure}

\subsection{Normal dependency in radiance field}
As presented in \citep{yariv2020multiview}, incorporating the zero level set normal in surface rendering improves both reconstruction quality and disentanglement of geometry and radiance. Incorporating the level set's normal has similar effect in the radiance field representation of VolSDF; see Figure \ref{fig:normal_ablation} where we compare using radiance field with and without normal dependency.

\begin{figure}[h]
    \includegraphics[width=\textwidth]{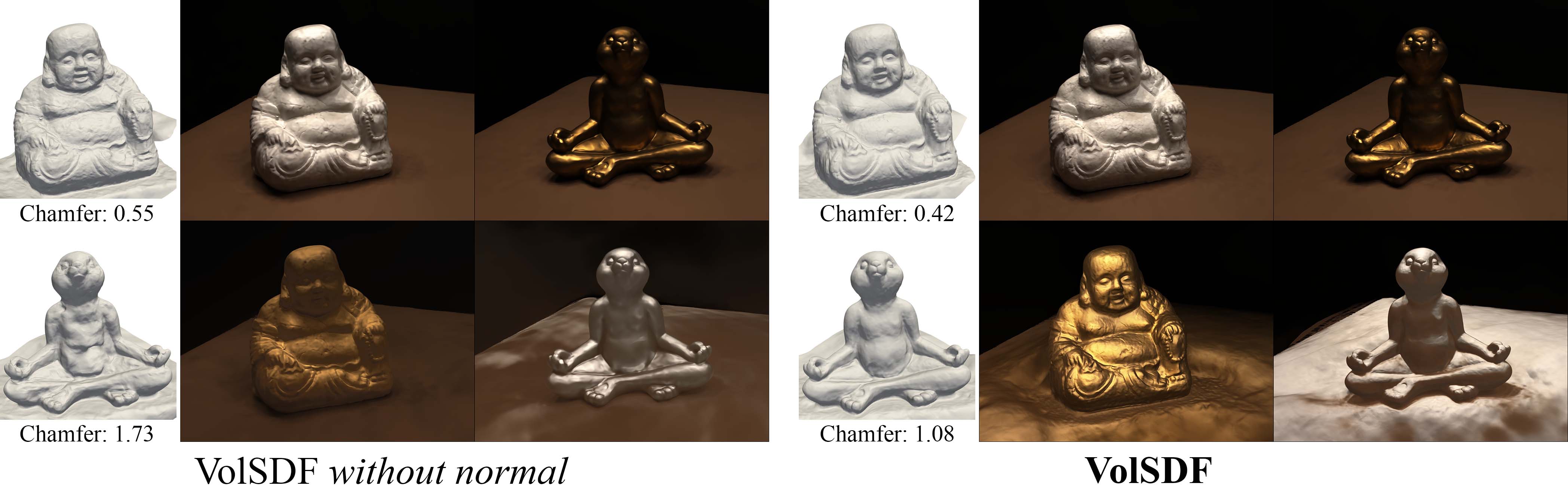}
    \caption{{Geometry and radiance field disentanglement is successful with normal information incorporated in the radiance field, and partially fails without it. }} \label{fig:normal_ablation}
\end{figure}

\subsection{Disentanglement of geometry and appearance}
Figure \ref{fig:dis_2} shows additional results for unsupervised disentanglement of geometry and radiance field (switching the radiance fields of three independently trained scenes). Note how the material and light from one scene is gracefully transferred to the two other scenes. We further provide in the supplementary \textbf{video} the 360 degrees camera path for one of the swaps.

\begin{figure}[h]
\centering
    \includegraphics[width=0.75\textwidth]{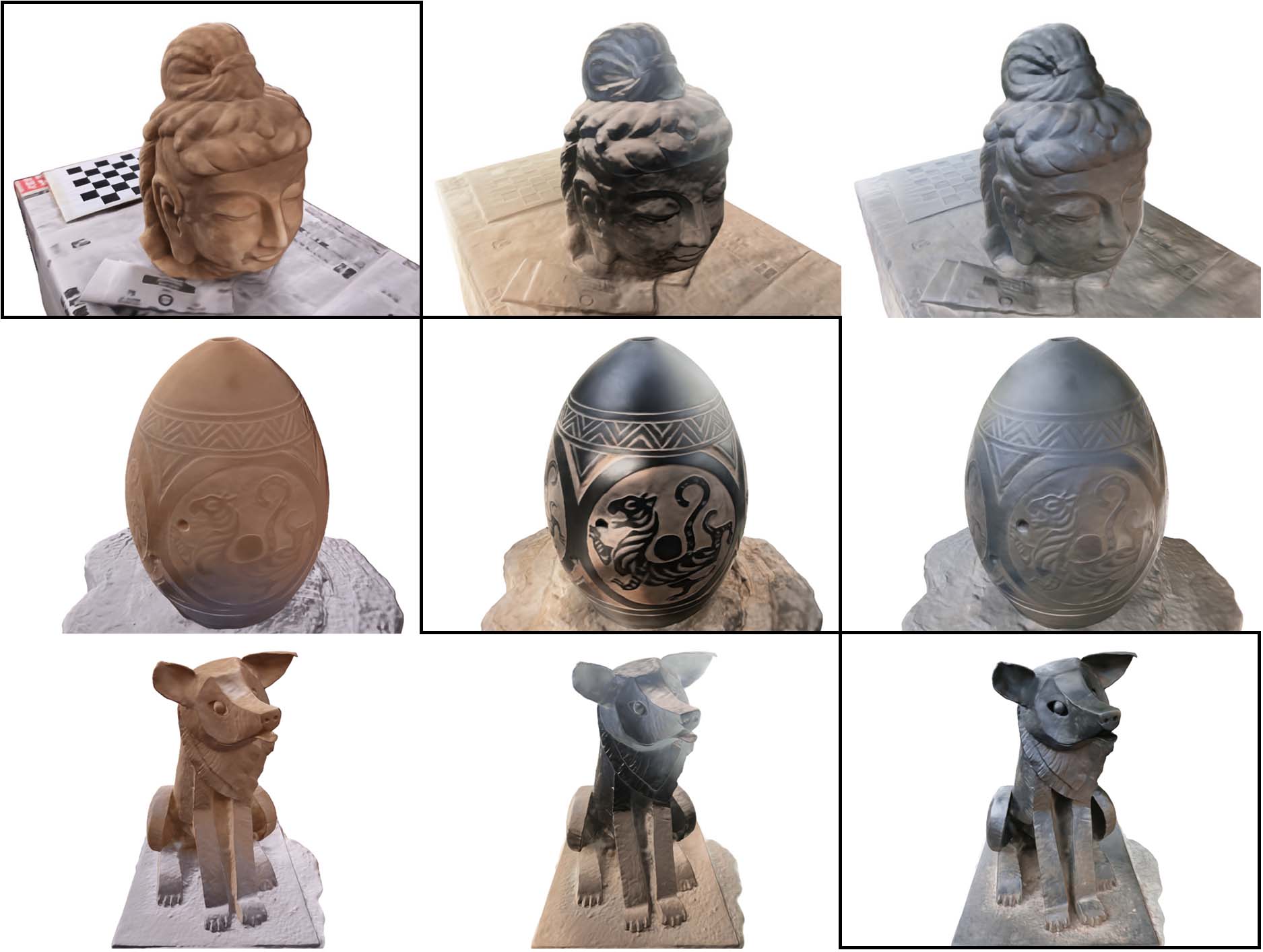}
    \caption{{Additional results of geometry and radiance field disentanglement in BlendedMVS. The diagonal depicts the trained geometry and radiance, while the off-diagonal demonstrates post-training mix and match of geometries and radiance fields. } } \label{fig:dis_2}
\end{figure}

\section{Experiments setup}
\subsection{Camera Normalization}\label{s:cam_norm}
We used the known camera poses to shift the coordinate system, locating the object at the origin. This is done using a least squares solution for the intersection point of all camera principal axes. Let $R_\textit{max}$ be the maximal camera center norm, we further apply a global scale of $\frac{3}{R_\textit{max}*1.1}$ to place all camera centers inside a sphere of radius $3$, centered at the origin.  

\subsection{Datasets}

\paragraph{DTU}
We used the formal evaluation script to measure the Chamfer $l_1$ distance between each reconstructed object to its corresponding ground truth point cloud. For fare comparison with  \cite{yariv2020multiview},  we used their masks (with a dilation of 50 pixels) to remove non visual hull parts from each output of both ours and \cite{mildenhall2020nerf}. Specifically, from an output mesh, we remove a 3D point (and its adjacent triangles) if it is projected to a zero pixel label in any image. Finally, we used the largest connected component mesh for evaluation. The results for COLMAP and IDR are taken from \cite{yariv2020multiview}.    

\paragraph{BlendedMVS}
We used the ground truth meshes supplied by the authors to evaluate the Chamfer $l_1$ distances from the output surfaces.  For each mesh we evaluated the largest connected surface component above the ground plane. To measure the Chamfer $l_1$ distance we used $100K$ random point samples from each surface. 

\subsection{Additional implementation details}
\paragraph{Architecture and hyper-parameters}
As described in section \ref{s:training}, our model architecture consists of two MLP networks, where the geometry network $\vf_\varphi$ has 8 layers with hidden layers of width 256, and a single skip connection from the input to the 4th layer. The geometry MLP gets a 3D position, $\vx$, and outputs the SDF value and an extra feature vector $\vz$ of size 256, \ie, $\vf_\varphi(\vx) = (d(\vx),\vz(\vx))\in\Real^{1+256}$. We initialized $\vf_\varphi$ using the geometric initialization presented in \cite{atzmon2019sal}, so that $d$ produces an approximated SDF to the unit sphere.
The radiance field network $L_\psi$ receives a 3D position $\vx$, the normal to its level set, $\vn$, and the (geometry) feature vector $\vz$, as well as the view direction $\vv$, and outputs a RGB value. It consists of 4 layers of width 256, and ends with a Sigmoid activation for providing valid color values. 
In addition to these two networks we have an additional scalar parameter $\beta$ that is initialized to $0.1$.

To capture the high frequencies of the geometry and radiance field, we exploit Positional Encoding (PE)  \citep{mildenhall2020nerf} for the position $\vx$ and view direction $\vv$ in the geometry and radiance field. For the position $\vx$ we use $6$ PE levels, and for the the view direction $\vv$ we use $4$ PE levels, as in \citep{yariv2020multiview}. The influence of different levels of frequencies applied to the position is presented in the Section \ref{ss:pe_ablation}.

\paragraph{Training details}
We trained our networks using the \textsc{Adam} optimizer \cite{kingma2014adam} with a learning rate initialized with $5\mathrm{e}{-4}$ and decayed exponentially during training to $5\mathrm{e}{-5}$. Each model was trained for $100K$ iterations on scenes from the DTU dataset, and for $200K$ on scenes from the BlendedMVS dataset. Training one model from the DTU dataset takes approximately $12$ hours. Training was done on a single Nvidia V-100 GPU, using \textsc{pytorch} deep learning framework \cite{paszke2017automatic}.

\paragraph{Mesh extraction} We use the Marching Cubes algorithm \cite{lorensen1987marching} for extracting each surface mesh from the zero level set of the signed distance function defined by $d(\vx)$.

\paragraph{Modeling the background} 
To satisfy the assumption that all rays, including rays that do not intersect any surface, are eventually occluded  (\ie, $O(\infty)=1$), 
we model our SDF as:
\begin{equation}\label{e:g_background}
d_\Omega(\vx) = \min \{ d(\vx), r - \norm{\vx}_2 \},
\end{equation}
where $r$ denotes a predefined scene bounding sphere (in out experiments $r=3$, see Section \ref{s:cam_norm} for more details about camera normalization).  During the rendering process, the maximal depth for ray samples, denoted by $M$ in Sec.~\ref{ss:volume_rendering_of_sigma}, is set to $2r$. Intuitively, this modification can be considered as modeling the background using a 360 degrees panorama on the scene boundaries. 
\\
For modeling more complex backgrounds as in the blendedMVS dataset, we follow the parametrization presented in NeRF++ \cite{kaizhang2020}: the volume outside a radius 3 sphere (the background) is modeled using an additional NeRF network, predicting the background point density and radiance field. Each $3D$ background point $(x_b,y_b,z_b)$ is modeled using a $4D$ representation $(x',y',z',\frac{1}{r})$ where $\norm{(x',y',z')}=1$ and $(x_b,y_b,z_b)=r\cdot(x',y',z')$. For rendering the background component of a pixel, we use $32$ points calculated by sampling $\frac{1}{r}$ uniformly in the range $(0,\frac{1}{3}]$.
More details regarding this parametrization (referred as the "inverted sphere parametrization") can be found in \cite{kaizhang2020}.

\subsection{Baselines}
\paragraph{NeRF} For running NeRF \cite{mildenhall2020nerf} we used a slightly modified version of the \textsc{pytorch} implementation suggested by \cite{lin2020nerfpytorch}. Following the official implementation code\footnote{\url{https://github.com/bmild/nerf}} of NeRF we extracted the $50$ level set of the learned density as NeRF geometry reconstruction. 

\paragraph{NeRF++} We used the official code\footnote{\url{https://github.com/Kai-46/nerfplusplus}} of NeRF++~\cite{kaizhang2020}. All the cameras are normalized inside a unit sphere, and we train two networks (one foreground network in normal coordinates, one background network in inverse sphere coordinates). For foreground, we used 64 uniform samples and 128 hierarchical samples; for background, we used 32 uniform and 64 hierarchical samples. Similar to NeRF experiments, we extracted the 50 level set of the foreground density as reconstruction.

\section{Proofs and additional lemmas}

As $d_\Omega$ is not everywhere differentiable we need to bound the Lipschitz constant of $\sigma$, rather than its derivative. The Lipshcitz constant is defined as a constant $K_i>0$ so that $\abs{\sigma(\vx(s))-\sigma(\vx(t))}\leq K_i |s-t|$, for all $s,t\in [t_i,t_{i+1}]$. 
\begin{theorem}
The Lipshcitz constant of the density $\sigma$ within a segment $[t_i,t_{i+1}]$ satisfies
\begin{equation}\label{e:dsigma_bound}
   K_i \leq \frac{\alpha}{2\beta}\exp\parr{-\frac{d^\star_i}{\beta}},\ 
    \text{where } d_i^\star = \min_{\substack{s\in [t_i,t_{i+1}]\\ \vy \notin B_i\cup B_{i+1}}} \norm{\vx(s)-\vy}
    ,\end{equation} 
    and $B_i=\set{\vx \ \vert \ \norm{\vx-\vx(t_i)}<|d_i|}$, $d_i=d_\Omega(\vx(t_i))$. 

\end{theorem}
The constant $d^*_i$ can be computed explicitly using $d_i, d_{i+1}, t_i, t_{i+1}$ as described in the next proposition.
\begin{proposition}\label{prop:dstar_der}
The lower distance bound $d^*_i$ can be computed using the following formulas:  
\begin{equation}\label{e:dsigma_bound}
    d_i^\star = \begin{cases}
    0 & |d_i| + |d_{i+1}| \leq \delta_i \\
    \min\set{|d_i|,|d_{i+1}|} & \abs{|d_i|^2 - |d_{i+1}|^2} \geq \delta^2_i \\
    h_i & \text{otherwise}
    \end{cases},
\end{equation}
$d_i=d_\Omega(\vx(t_i))$, $h_i = \frac{2}{\delta_i}\sqrt{s(s-\delta_i)(s-|d_i|)(s-|d_{i+1}|)}$,  $s=\frac{1}{2}(\delta_i+|d_i|+|d_{i+1}|)$, $\delta_i=t_{i+1}-t_i$.
\end{proposition}

\begin{proof}[Proof of Theorem \ref{thm:density_der}.]
We denote by $\Phi_\beta$ the Probability Density Function (PDF) of the Laplace distribution (the CDF of which is given in \eqref{e:laplace}), 
\begin{equation}\label{e:laplace_pdf}
\frac{d}{dt}\Psi_\beta(s)=\Phi_\beta(s)=\frac{1}{2\beta}\exp\parr{-\frac{\abs{s}}{\beta}}.
\end{equation}

Let $s,t\in[t_{i},t_{i+1}]$,
\begin{align*}
\abs{\sigma(\vx(s))-\sigma(\vx(t))} &= \alpha\abs{\Psi_\beta(-d_\Omega(\vx(s)))-\Psi_\beta(-d_\Omega(\vx(t)))}\\ & \leq \alpha \abs{d_\Omega(\vx(s))-d_\Omega(\vx(t))}\max_{t\in[t_i,t_{i+1}]}\abs{\Phi_\beta(-d_\Omega(\vx(t))) } \\ &\leq \alpha \abs{s-t} \Phi_\beta\parr{\min_{t\in [t_i,t_{i+1}]} \abs{d_\Omega(\vx(t))}}
\\ &\leq \alpha \abs{s-t} \Phi_\beta(d^*_i)
\end{align*}
where the first equality is by definition of $\sigma$ (\eqref{e:density}), the first inequality uses the fact that the Lipschitz constant of a continuously differential function is the maximum of (absolute value of) its derivative. The second inequality uses the following Lemma \ref{lem:lipschitz_d} and the fact that $\Phi_\beta$ is symmetric w.r.t.~zero, positive, and monotonically decreasing at $[0,\infty)$. The last inequality can be justified by noting that by definition of the distance function $\gM \subset (B_i \cup B_{i+1})^c$ and therefore $\min_{t\in [t_i,t_{i+1}]} \abs{d_\Omega(\vx(t))}\geq d^*_i$.
\end{proof}

\begin{proof} [Proof of Proposition \ref{prop:dstar_der}.]
We wish to find the distance of the two sets: $A=[\vx(t_i),\vx(t_{i+1})]$ and $B=\partial ((B_i \cup B_{i+1})^c)$. That is $d^*_i = \min_{\va\in A,\vb\in B}\norm{\va-\vb}$.  Since these two sets are compact, their cartesian product is compact, and the minimum is achieved. Denote this minimum $(\va^*,\vb^*)$, where $\va^*\in A $ and $\vb^*\in B$. 

First, if $B_i\cap B_{i+1}=\emptyset$ then $d^*_i=0$. This case is characterized by $|d_i|+|d_{i+1}| \leq \delta_i$. Henceforth we assume $|d_i|+|d_{i+1}| > \delta_i$.

Now we consider several cases. If $\va^*= \vx(t_i)$ the minimal distance is $|d_i|$, and similarly the distance is $|d_{i+1}|$ if $\va^*=\vx(t_{i+1})$. 
Otherwise, if $\va^*\in(\vx(t_i),\vx(t_{i+1}))$ Lagrange Multipliers imply that $\va^*-\vb^* \perp \vv$, namely is orthogonal to the ray. 
In this case we have two options to consider: First,  $\vb^*\in \bar{B}_i\cap \bar{B}_{i+1}$, where $\bar{B}_i$ is the closure of the ball $B_i$. In this case Heron's formula provides the minimal distance $h_i$, as the height of the triangle $\Delta(\vx(t_i),\vb^*,\vx(t_{i+1}))$.

Otherwise $\vb^*\in \bar{B}_i$ or $\vb^* \in \bar{B}_{i+1}$. Lets assume the former (the latter is treated similarly). Lagrange multipliers imply that $\vb^* - \va^* \| \vb^* - \vx(t_i) $. This necessarily means that $\va^*=\vx(t_i)$, leading to a contradiction with the assumption that $\va^*\in(\vx(t_i),\vx(t_{i+1}))$. 

To check if $\va^*\in(\vx(t_i),\vx(t_{i+1}))$ and $\vb^*\in \bar{B}_i\cap \bar{B}_{i+1}$ it is enough to check that the angles at $\vx(t_i)$ and $\vx(t_{i+1})$ in the triangle $\Delta(\vx(t_i),\vb^*,\vx(t_{i+1}))$ are acute, or equivalently that $\abs{|d_i|^2 - |d_{i+1}|^2} < \delta^2_i$ (the latter can be shown with the Law of Cosines).

\end{proof}

\renewcommand{\thelemma}{A\arabic{lemma}}
\setcounter{lemma}{0}

\begin{lemma}\label{lem:lipschitz_d}
The distance function $d=d_\Omega$ to a compact surface $\gM$ is Lipschitz with constant $1$, \ie, $\abs{d(\vx)-d(\vy)}\leq \norm{\vx-\vy}$, for all $\vx,\vy$. 
\end{lemma}
\begin{proof}[Proof of Lemma \ref{lem:lipschitz_d}.]
Let $\vx^*\in\gM$ be the closest point to $\vx$, and $\vy^*\in\gM$ the closets point to $\vy$. If $d(\vx)d(\vy)<0$, then $\abs{d(\vx)-d(\vy)} = \norm{\vx-\vx^*}+\norm{\vy-\vy^*} \leq \norm{\vy-\vz} + \norm{\vz - \vx}$, for all $\vz\in\gM$. However, from mean value theorem we know there exists $\vz\in\gM$ on the straight line $[\vx,\vy]$, and therefore $\abs{d(\vx)-d(\vy)} \leq \norm{\vx-\vy}$. If $d(\vx)d(\vy)\geq 0$, then $\abs{d(\vy)-d(\vx)} = \abs{|d(\vy)|-|d(\vx)|}$. Now, 
\begin{align*}
 |d(\vy)|&= \norm{\vy-\vy^*}\leq \norm{\vy-\vx^*} \\
& \leq \norm{\vy-\vx} + \norm{\vx-\vx^*} \\
&= \norm{\vx-\vy} + |d(\vx)|.
\end{align*}
Therefore $|d(\vy)|-|d(\vx)| \leq \norm{\vx-\vy}$. The other direction follows by switching the roles of $\vx,\vy$. We proved $\abs{d(\vx)-d(\vy)}\leq \norm{\vx-\vy}$.
\end{proof}

\renewcommand{\thelemma}{\arabic{lemma}}
\setcounter{lemma}{0}

\begin{proof}[Derivation of \eqref{e:E_bound}.]
For a single interval $[t_k,t_{k+1}]$,
\begin{align*}
  \abs{\int_{t_i}^{t_{i+1}}\sigma(\vx(s))ds - \delta_i\sigma_i} &\leq \int_{t_i}^{t_{i+1}} \abs{ \sigma(\vx(s)) - \sigma(\vx(t_i)) } ds \leq 
  K_i
  \frac{\delta_i^2}{2}
\end{align*}
Plugging the Lipschitz bound from Theorem \ref{thm:density_der} (\eqref{e:dsigma_bound}) provides that the error in \eqref{e:rectangle_rule} for $t\in[t_k,t_{k+1}]$ can be bounded by 
\begin{equation}\label{e:E_bound_supp}
\abs{E(t)}\leq \widehat{E}(t) = \frac{\alpha}{4\beta}\parr{\sum_{i=1}^{k-1} \delta_i^2 e^{-\frac{d^\star_i}{\beta}} + (t-t_k)^2 e^{-\frac{d^\star_k}{\beta}}}.
\end{equation}
\end{proof}

\begin{theorem}\label{thm:bound}
For $t\in[0,M]$, the error of the approximated opacity $\hat{O}$ can be bounded as follows:
\begin{equation}\label{e:bound_O}
\abs{O(t)-\widehat{O}(t)} \leq \exp\parr{-\widehat{R}(t)}\parr{\exp\parr{\widehat{E}(t)}-1} 
\end{equation}
\end{theorem}

\begin{proof}[Proof of Theorem \ref{thm:bound}.]
This bound is derived as follows:
\begin{align*}
    \abs{O(t)-\widehat{O}(t)} & = \abs{\exp\parr{-\widehat{R}(t)} - \exp\parr{-\int_0^t \sigma(\vx(s))ds}} = \exp\parr{-\widehat{R}(t)}\abs{1-\exp\parr{-E(t)}} \\
    &\leq \exp\parr{-\widehat{R}(t)} \parr{ \exp(\widehat{E}(t)) -1 },
\end{align*}
where the last inequality is due to the inequality $\abs{1-\exp(r)}\leq \exp(|r|)-1$ and the bound $|E(t)|\leq \widehat{E}(t)$ from \eqref{e:E_bound}. 
\end{proof}

\begin{lemma}\label{lem:dense}
Fix $\beta>0$. For any $\eps>0$ a sufficient dense sampling $\gT$ will provide $B_{\gT,\beta}<\eps$.
\end{lemma}

\begin{proof}[Proof of Lemma \ref{lem:dense}.] 
Noting that $\exp(-\widehat{R}(t))\leq 1$, \eqref{e:bound} leads to 
\begin{align*}
  B_{\gT,\beta}\leq \parr{\exp\parr{\widehat{E}(t_n)}-1} &= \parr{\exp\parr{\frac{\alpha}{4\beta}\sum_{i=1}^{n-1} \delta_i^2 e^{-\frac{d^\star_i}{\beta}}} -1} \\
  & \leq \parr{\exp\parr{\frac{\alpha}{4\beta}\sum_{i=1}^{n-1} \delta_i^2 } -1}.
\end{align*}
Lastly we note that the inner sum satisfies $\sum_i \delta_i^2 \leq M \max_i \delta_i$, and in turn this implies that dense sampling can achieve arbitrary low error bound.
\end{proof}

\begin{lemma}\label{lem:beta_plus}
Fix $n>0$. For any $\eps>0$ a sufficiently large $\beta$ that satisfies
\begin{equation}
        \label{e:beta_0}
  \beta \geq  \frac{\alpha M^2}{4(n-1)\log(1+\eps)} 
    \end{equation}   
    will provide $B_{\gT,\beta}\leq \eps$.
\end{lemma}

\begin{proof}[Proof of Lemma \ref{lem:beta_plus}.]
Assuming sufficiently large $\beta$ that satisfies \eqref{e:beta_0}, then following the derivation of Lemma \ref{lem:dense} and assuming equidistance sampling, \ie $\delta_i = \frac{M}{n-1}$, leads to
\begin{align*}
 B_{\gT,\beta} &\leq \parr{\exp\parr{\frac{\alpha}{4\beta}\sum_{i=1}^{n-1} \delta_i^2 } -1} = \parr{\exp\parr{\frac{\alpha M^2}{4(n-1)\beta}} -1} \\ & \leq \parr{\exp\Big(\log(1+\eps)\Big) -1}  = \eps
\end{align*}
    
\end{proof}

\section{Numerical integration}
Standard numerical quadrature in volume rendering is based on the rectangle rule \citep{max1995optical}, which we repeat here for completeness. Quantities with $\hat{\ }$ denote approximated quantities. 
For a set of discrete samples $\gS=\set{s_i}_{i=1}^m$, $0=s_1<s_2<\ldots<s_m=M$ we let $\delta_i=s_{i+1}-s_{i}$, $\sigma_i=\sigma(\vx(s_i))$, and $L_i=L(\vx(s_i),\vn(s_{i}),\vv)$. 
Applying the rectangle rule (left Riemann sum) to approximate the integral in \eqref{e:vol} we have 
\begin{equation}\label{e:rect_1}
\begin{aligned}
I(\vc,\vv) &= \int_0^\infty  L(\vx(t),\vn(t),\vv) \tau(t) dt \\ &= 
\int_{0}^{M} L(\vx(t),\vn(t),\vv) \tau(t) dt + \int_M^\infty L(\vx(t),\vn(t),\vv) \tau(t) dt \\ 
& \approx
\sum_{i=1}^{m-1} \delta_i \tau(s_i) L_i,
\end{aligned}
\end{equation}
where we assumed that $M$ is sufficiently large so that the integral over the segement $[M,\infty)$ is negligible, and the integral over $[0,M]$ is approximated with the rectangle rule, and $\tau(s_i) = \sigma_i T(s_i)$ as given in \eqref{e:tau}. The rectangle rule is applied yet again to approximate the transparency $T$:
\begin{align*}
T(s_i)\approx \hat{T}(s_i)=\exp \parr{-\sum_{j=1}^{i-1} \sigma_j\delta_j}.
\end{align*}
To provide a discrete probability distribution $\hat{\tau}$ that parallels the continuous one $\tau$, \cite{max1995optical} defines $p_i=\exp(-\sigma_i\delta_i)$ that can be interpreted as the probability that light passes through the segment $[s_i,s_{i+1}]$. Then, using $\hat{T}(s_i)=\prod_{j=1}^{i-1}p_i$, and the approximation $\delta_i \sigma_i \approx (1-\exp(-\delta_i \sigma_i))=1-p_i$, the approximation in \eqref{e:rect_1} takes the form 
\begin{align*}
I(\vc,\vv) \approx \hat{I}(\vc,\vv) = \sum_{i=1}^{m-1}\brac{ (1-p_i) \prod_{j=1}^{i-1}p_j  } L_i= \sum_{i=1}^{m-1} \hat{\tau}_i L_i 
\end{align*}
where the discrete probability is given by
\begin{align*}
 \hat{\tau}_i =\parr{1-p_i}\prod_{j=1}^{i-1}p_j \quad \text{for }\ 1\leq i\leq m-1, \ \text{ and  }\  \hat{\tau}_m=\prod_{j=1}^{m-1}p_j
\end{align*}
establishing that $\hat{\tau}_i$, $i=1,\ldots,m$ is indeed a discrete probability. Note that although $\hat{\tau}_m$ is not used in $\hat{I}$ above, we added it for implementation ease. Since we use a bounding sphere (see \eqref{e:g_background}) and $M$ is chosen so that $s_m$ is always outside this sphere, $\hat{\tau}_m\approx 0$.

\end{document}